%% file: neurips_2022.tex
\documentclass{article}

\usepackage[final]{neurips_2022}

\usepackage[utf8]{inputenc} 
\usepackage[T1]{fontenc}    
\usepackage{url}            
\usepackage{booktabs}       
\usepackage{amsfonts}       
\usepackage{nicefrac}       
\usepackage{microtype}      
\usepackage{xcolor}         
\usepackage{algorithm}
\usepackage{amsmath}
\usepackage{amssymb}
\usepackage{mathtools}
\usepackage{amsthm}
\usepackage[numbers]{natbib}
\usepackage{algpseudocode}
\usepackage{wrapfig}

\usepackage{fleqn, tabularx}
\usepackage{multirow}
\usepackage[export]{adjustbox}
\usepackage{colortbl}
\usepackage{multirow,tabularx,booktabs}

\theoremstyle{plain}
\newtheorem{theorem}{Theorem}[section]
\newtheorem{proposition}[theorem]{Proposition}
\newtheorem{lemma}[theorem]{Lemma}

\theoremstyle{definition}
\newtheorem{definition}[theorem]{Definition}
\newtheorem{assumption}[theorem]{Assumption}
\theoremstyle{remark}
\newtheorem{remark}[theorem]{Remark}

\usepackage[colorlinks=true,linkcolor=blue,citecolor=blue]{hyperref}
\usepackage[skins,theorems]{tcolorbox}

\usepackage[capitalize,noabbrev]{cleveref}

\usepackage[textsize=tiny]{todonotes}

\title{Data-Driven Offline Decision-Making via\\ Invariant Representation Learning}

\author{%
  Han Qi$^*$, Yi Su$^*$, Aviral Kumar$^*$, Sergey Levine\\
  Department of Electrical Engineering and Computer Sciences, UC Berkeley \\
  \texttt{\{han2019, aviralk\}@berkeley.edu, yisumtv@google.com}~~~ ($^*$Equal Contribution)
}

\begin{document}
\include{defs}
\maketitle

\begin{abstract}
The goal in offline data-driven decision-making is synthesize decisions that optimize a black-box utility function, using a previously-collected static dataset, with no active interaction. These problems appear in many forms: offline reinforcement learning (RL), where we must produce actions that optimize the long-term reward, bandits from logged data, where the goal is to determine the correct arm, and offline model-based optimization (MBO) problems, where we must find the optimal design provided access to only a static dataset. A key challenge in all these settings is distributional shift: when we optimize with respect to the input into a model trained from offline data, it is easy to produce an out-of-distribution (OOD) input that appears erroneously good. In contrast to prior approaches that utilize pessimism or conservatism to tackle this problem, in this paper, we formulate offline data-driven decision-making as \emph{domain adaptation}, where the goal is to make accurate predictions for the value of optimized decisions (``target domain''), when training only on the dataset (``source domain''). This perspective leads to invariant objective models (\methodname), our approach for addressing distributional shift by enforcing invariance between the learned representations of the training dataset and optimized decisions. In \methodname, if the optimized decisions are too different from the training dataset, the representation will be forced to lose much of the information that distinguishes good designs from bad ones, making all choices seem mediocre. Critically, when the optimizer is aware of this representational tradeoff, it should choose not to stray too far from the training distribution, leading to a natural trade-off between distributional shift and learning performance.
\end{abstract}

\vspace{-0.27cm}
\section{Introduction}
\vspace{-0.27cm}
Many real-world applications of machine learning involve learning how to make better decisions from data. When we must make a sequence of decisions (e.g., control a robot), this can be formulated as offline reinforcement learning (RL)~\citep{lange2012batch,levine2020offline}, whereas in problems where we must synthesize only one decision (e.g., a molecule that can effectively catalyze a reaction), this can be formulated as a multi-armed bandit or a data-driven model-based optimization (also known as black-box optimization~\citep{angermueller2020population}). In all of these cases, an algorithm is provided with data from previous experiments consisting of $(\bx,y)$ tuples, where $\bx$ represents a decision (e.g., a bandit arm, a design such as a molecule or protein, or a sequence of actions in RL) and $y$ represents its utility (e.g., certain property of the molecule or the long-term reward of a trajectory). The goal is to generate the best possible decision, typically one that is better than the best one in the dataset.
All of these problem domains are united by a common challenge: in order to improve, the decisions learned by the algorithm must differ from the distribution of the designs in the training data, inducing \emph{distributional shift}~\citep{kumar2019model,kumar2019stabilizing}, which must be carefully handled.

Current methods for solving such data-driven decision-making problems try to learn some form of a proxy model $\hat{f}(\cdot)$ that maps a decision $\bx$ to its objective value $y$ via supervised learning on the static dataset, and then optimize the decision $\bx$ against this learned model (in reinforcement learning, the proxy model is typically the action-value function and is learned via Bellman backups, thanks to the Markovian structure). However, overestimation errors in the learned model will erroneously drive the optimizer towards low-value designs that ``fool'' the proxy into producing a high values~\cite{kumar2019model}. To handle this, prior works have proposed to regularize the learned model to be conservative by penalizing its values on unseen designs~\citep{kumar2020conservative,trabucco2021conservative,yu2021roma,jin2020pessimism}, or explicitly constraining the optimizer (e.g., by using density estimation~\citep{brookes19a,kumar2019model}, referred to as ``policy constraints'' in offline RL~\citep{levine2020offline}).

In this paper, we take a different perspective for the distributional shift challenges in offline data-driven decision-making, and formulate it as domain adaptation, where the ``target'' domain consists of the distribution of decisions that the optimization procedure finds, and the source domain corresponds to the training distribution. Framing the problem in this way, we derive a class of methods that regulate distributional shift at the \emph{representation level}, and admit effective hyperparameter tuning strategies. The intuition behind our approach is that, if the internal representations of the learned model are made \emph{invariant} to the differences between the decisions seen in the training data vs. the optimized distribution, then excessive distributional shift is naturally discouraged because excessive invariance would lead the out-of-distribution points to attain mediocre values on average. That is, as the optimized decisions deviate more from the training data, the requirement to maintain invariance will cause the model's representation to lose information, and will hence force the predicted value to become closer to the average $y$ value in the dataset. This will force the optimizer to stay within the distribution of the dataset as recall that out-of-distribution points will only appear mediocre under the learned model. Of course, the representation cannot be perfectly invariant in general, because the optimized points are different from the training data, and hence the method must strike a balance. In practice, one can instantiate this idea using an adversarial training procedure. 

\begin{wrapfigure}{r}{0.45\textwidth}
\vspace{-0.75cm}
    \includegraphics[width=0.45\textwidth]{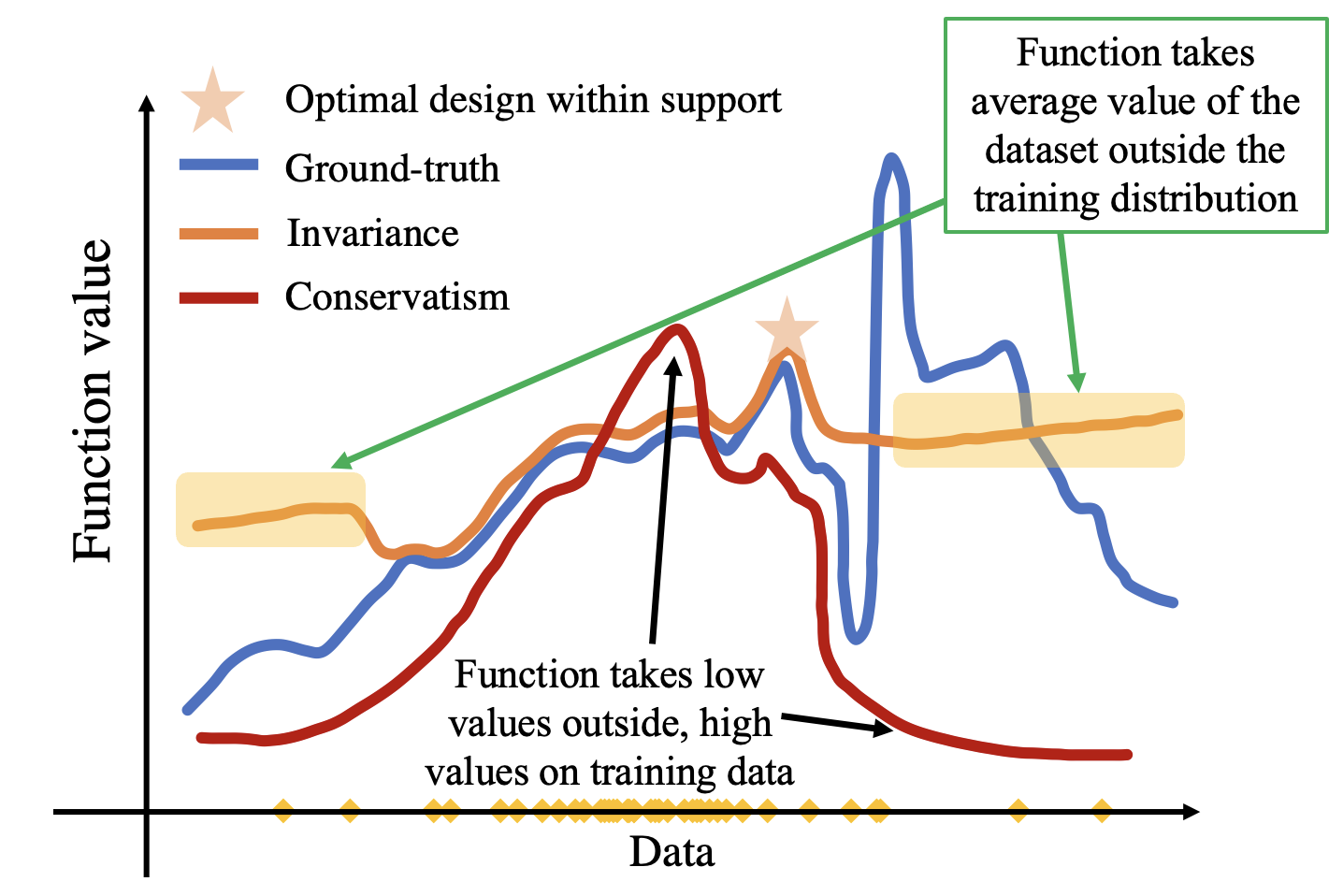}
    \vspace{-0.2cm}
    \hspace*{-1.3cm}\caption{\footnotesize \textbf{Illustration showing the intuition behind \methodname} on a 1D task. While conservatism (COMs)~\citep{trabucco2021conservative} pushes down the learned values on designs (decisions) not observed in the dataset and pushes up the learned values on designs (decisions) in the dataset, training with invariance induces the out-of-distribution designs (decisions) to attain objective values close to the average value in the training dataset. This discourages the optimizer from finding out-of-distribution designs (decisions).}
    \vspace{-0.3cm}
\end{wrapfigure}
Our main contribution is a new class of methods for offline data-driven decision-making, derived from a novel connection with domain adaptation. In this paper, we instantiate this idea in the setting of offline data-driven model-based optimization (MBO) (i.e., offline data-driven bandits), leaving the sequential offline reinforcement learning setting to future work.  As discussed above, our approach, invariant objective models (\methodname), optimizes the decision -- in this case, a ``design'' -- against a learned model that admits invariant representations between optimized designs and the dataset. This enables controlling distributional shift and also provides a way to derive effective offline hyperparameter tuning strategies that we show actually work well in practice. \methodname\ can be implemented simply by combining any supervised regression procedure with a regularizer that enforces invariance. The regularizer is implemented using any discrepancy measure between two distributions, and we instantiate \methodname\ using the $\chi^2$-discrepancy measure via the least-squares generative adversarial network.
This procedure trains a discriminator to discriminate between the \emph{representations} on the training dataset and optimized designs, whereas the learned representations are trained to be as invariant as possible to fool the discriminator. We theoretically derive \methodname\ from our domain adaptation formulation of MBO and we also use this formulation to derive tuning strategies for several common hyperparameters.  Empirically, we evaluate \methodname\ on several tasks from Design-Bench~\citep{trabucco2021designbench} and find that it outperforms the best prior methods, and additionally, admits appealing offline  tuning strategies unlike the prior methods.

\vspace{-0.3cm}
\section{Preliminaries}
\label{sec:prelims}
\vspace{-0.3cm}

\textbf{Offline data-driven decision-making and its variants.} The goal in offline data-driven decision-making is to produce decisions that maximize some objective function, using experience from a provided static dataset. Two instances of this problem are offline reinforcement learning (RL), and offline data-driven model-based optimization (MBO). As shown in Table~\ref{tab:terminology}, in offline RL, the decision is a policy $\pi$ that prescribes an action $\ba$ for any state $\bs$, and the objective function is the long-term discounted reward attained by the policy $\pi$, denoted as $J(\pi)$. In offline MBO, the decision corresponds to a design ($\bx$), and the objective function is the unknown black-box utility ($f(\bx)$). We will utilize the notation from offline MBO in this paper for simplicity, but the ideas developed in this paper can be extended to the RL setting.     

\begin{table}[t]
\label{tab:terminology}
 \vspace{-0.4cm}
  \centering
  \footnotesize
  \def\arraystretch{0.9}
  \setlength{\tabcolsep}{0.42em}
\begin{tabularx}{0.95\linewidth}{c X X}
\toprule
\textbf{Quantity} & \textbf{Offline RL} & \textbf{Offline MBO}\\
\midrule
Decision & a policy $\pi$ mapping states to actions & a design $\bx$ (e.g., a molecule)  \\
Objective function & long-term discounted reward, $J(\pi) = \mathbb{E}_{\ba_{0:\infty} \sim \pi}[\sum_{t} \gamma^t r(\bs, \ba_t)]$ & the objective value, $f(\bx)$\\
\midrule
Training dataset & a dataset $\mathcal{D} = \{(\tau_i, r_i) \}_{i=1}^N$ sampled from a dist. over rollouts $\tau$, $\mu(\tau)$ & a dataset $\mathcal{D} = \{(\bx_i, y_i) \}_{i=1}^N$ sampled from the distribution of designs, $\mu(\bx)$  \\
Optimization problem & using $\mathcal{D}$ find $\arg \max_{\pi}~ J(\pi) $ & using $\mathcal{D}$ find $\arg \max \mathbb{E}_{\bx \sim \Opt}[f(\bx)]$ \\
\midrule
Proxy model & learn a model of $J(\pi)$, typically in the form of a value function, $\widehat{Q}(\bs, \ba)$ & learn a model of $f$, $\widehat{f}(\bx)$ \\
A typical method & optimize the learned Q-function $\widehat{Q}$ & optimize the learned model $\widehat{f}$ \\
\midrule
Distributional shift & shift between $\mu(\ba|\bs)$ and $\pi(\ba|\bs)$ & shift between $\mu(\bx)$ and $\Opt(\bx)$ \\
\bottomrule
    \end{tabularx}
    \vspace{0.05cm}
         \caption{\label{tab:terminology} \footnotesize{\textbf{Instantiation of our generic formulation of offline decision-making} in the context of offline RL and offline MBO. While the rest of the paper will adopt the terminology from offline MBO, the ideas we present in this paper can also be applied to offline RL. 
     }}
\vspace{-0.5cm}
\end{table}

\textbf{Offline model-based optimization (MBO).} As mentioned above, the goal in offline model-based optimization (MBO)~\citep{trabucco2021designbench} is to produce designs that maximize some objective function $f(\bx)$, by only utilizing a dataset $\mathcal{D} = \{(\bx_i, y_i)\}_{i=1}^n$ of designs $\bx_i$ and their corresponding objective values $y_i$. No active queries to the groundtruth function are allowed. Typically offline MBO is formulated as finding the optimal design $\bx^\star$. However, we will use a slightly more general formulation, where the goal is to optimize a \emph{distribution} over designs $\mu(\bx)$. We will use this formulation largely for convenience of analysis, where the classic case is recovered when $\mu(\bx)$ is a Dirac delta function, though we also note that optimizing distributions over designs is often studied in settings such as synthetic biology, where the goal is to explicitly diversify the designs~\citep{biswas2021low}. Thus, our goal is to find a distribution over designs $\Opt(\bx)$ using only the static dataset $\mathcal{D}$, such that it maximizes the \emph{expected value} of $f(\bx)$ (the expected value under distribution $\mu$ is referred to as $J(\mu)$). This can be formalized as follows ($\mathsf{Alg}$ refers to an algorithm): 
\begin{align}
    \Opt := \arg \max_{\mathsf{Alg}}~ \bbE_{\bx \sim \mu(\bx)}[f(\x)] := J(\mu)~~~~ \text{s.t.}~~ \mu = \mathsf{Alg}(\mathcal{D}). \vspace{-0.1cm}
\label{eqn:mbo}
\end{align}
\textbf{Model-based optimization methods.} The methods we consider in this paper fit a learned model ${f}_\theta(\bx)$ to $(\bx_i, y_i)$ pairs in the dataset via a standard supervised regression objective, with an additional regularizer.
We will assume that ${f}_\theta$ is composed of two components: a representation $\phi \in \mathbb{R}^d$ that maps a given $\bx$ into a $d$-dimensional vector, and a forward model, ${f}_\theta$ that converts the representational output into a scalar objective value, i.e., ${f}_{\phi}(\bx) \coloneqq {f}_\theta(\phi(\bx))$. We will use the notation $f_\theta$ and $f_\phi$ interchangeably. 
Given a learned model ${f}_\theta(\phi(\bx))$, one approach to obtain the optimized distribution $\Opt$ is by optimizing the expected value ${J}_\theta(\mu) := \bbE_{\bx \sim \mu}[{f}_\theta(\phi(\bx))]$ under the learned model. This distribution $\Opt(\bx)$ can be instantiated in many ways: while some prior work~\citep{kumar2019model,brookes19a} represents $\Opt$ via a parametric density model over $\mathcal{X}$ and optimizes this density model to obtain $\Opt$, other prior work~\citep{trabucco2021conservative} uses a non-parametric representation for $\Opt$ using a set of optimized ``design particles''. 

In this paper, we represent $\Opt$ using a set of design particles, each of which is obtained by optimizing the learned function $f_\theta(\phi(\bx))$ with respect to the input $\bx$ via grad. ascent, starting from \emph{i.i.d.} samples from the dataset. Formally:
\begin{equation}
\label{eqn:gradient_ascemt}
\x^i_{k+1} \leftarrow \x^i_{k} + \eta \nabla_\bx f_\theta(\phi(\bx))\vert_{\bx = \x^i_k}, \text{~~for~~} k \in [1, T], ~~~ \x^i_0 \sim \mathcal{D}   
\end{equation}

The optimized distribution $\Opt$ is then given by $\Opt = \delta\{\x^1_T, \x^2_T, \cdots\}$. 

\textbf{The challenge of distributional shift in offline MBO}~\citep{kumar2019model,brookes2019conditioning,fannjiang2020autofocused,trabucco2021conservative}. When the learned model $f_\theta(\phi(\bx))$ is trained via supervised regression on the offline dataset $\mathcal{D}$: $\arg\min_{\theta} \frac{1}{n}\sum_{i=1}^n({f}_\theta(\phi(\bx_i)) - y_i)^2$, the ERM principle states that $f_\theta(\phi(\bx))$ will reasonably approximate $f(\bx)$ in expectation only under the distribution of the training dataset $\mathcal{D}$, i.e., $\mu_\text{data}$, however $f_\theta(\phi(\bx_i)) \neq y_i$ in general on points not in the training distribution. Then, when optimizing the learned function, \emph{overestimation errors} in the value of $f_\theta(\phi(\bx))$ will inevitably lead the optimizer to converge to a distribution $\Opt$ over designs that have a high value under the learned function(i.e., high $J_\theta(\Opt)$) but not a high ground truth value $J(\Opt$). In general, this distribution $\Opt$ will be different from $\mu_\text{data}$, since the learned function is likely to be more correct in expectation under $\mu_\text{data}$.

\vspace{-0.23cm}
\section{Data-Driven Offline Model-based Optimization as Domain Adaptation}
\label{sec:method}
\vspace{-0.23cm}

In this section, we will formulate data-driven offline MBO as domain adaptation, which will allow us to derive our method, invariant objective models (\methodname). Instead of attempting to simply constrain the optimizer to directly avoid distributional shift, our approach modifies the representations learned by the model $\hat{f}_\theta$ by adding an invariance regularizer during training such that there is minimal distribution shift under the learned invariant representation during optimization.

If we can ensure that the learned model $f_\theta$ is accurate on the distribution found by the optimizer, $\Opt$, then we can ensure that we will obtain good designs. Therefore, to handle the distributional shift, we can treat $\mu_\text{data}$ as the ``source domain'' and the optimized distribution $\Opt$ as the ``target domain''. Just as the goal in domain adaptation is to train on the source domain and make accurate predictions on the target domain, we will aim to train $f_\theta(\bx)$ on $\mu$ and with the goal of making more accurate predictions under $\Opt$. Of course, learning to make accurate predictions on \emph{any} unseen distribution is impossible for any domain adaptation method~\citep{zhao2019learning}, but in the case of MBO, the choice of the target distribution $\Opt$ is also made by the algorithm itself, such that invariance can be enforced \emph{both} by changing the representation and changing the target distribution $\Opt$.

\vspace{-0.2cm}
\subsection{Invariant Objective Models (\methodname)}
\label{sec:mmd_method}
\vspace{-0.2cm}

\methodname\ controls distributional shift between the optimized distribution $\Opt$ and the training distribution $\mu_\text{data}$ at the representational level, by regularizing the distribution of the features $\phi(\bx)$ under $\Opt$ to match that under $\mu_\text{data}$. We will first present the intuition behind the method, and then formalize it by setting up a bi-level optimization problem that we will also analyze theoretically.

\textbf{Intuition:} As the optimizer deviates from the training distribution $\mu$, the invariance regularizer forces the distribution of optimized designs to resemble the distribution over representations $\phi(\bx)$ under $\mu_\text{data}$,
i.e.,  $\bbP_{\mu_{\Optold}}(\phi(\bx)) \approx \bbP_{\mu_\text{data}}(\phi(\bx))$. When $\mu_{\Optold}$ is far from the training distribution (i.e., if $D_\mathrm{KL}(\mu_{\Optold}, \mu_\text{data})$ is very high), it is easier to enforce invariance as it does not conflict with fitting $f_\theta$ on the training distribution. 
This representational invariance in turn causes the expected value of the learned function $f_\theta$ under $\Opt$ to appear (roughly) equal to the average objective value in the training dataset, i.e., $\mathbb{E}_{\bx \sim \Opt}[{f}_\theta(\phi(\bx))] \approx \mathbb{E}_{\bx \sim \mu_\text{data} }[{f}_\theta(\phi(\bx))]$, if $\Opt$ is too far away from the training distribution.
As a result, the expected value of the learned model $f_\theta(\phi(\bx))$ under $\Opt$, $J_\theta(\Opt)$, will be worse than the value of the best design within the training distribution (since the best design input sampled from $\mu_\text{data}$ will attain a higher objective value than the average value under $\mu_\text{data}$). This would act as a guard to prevent the optimizer diverging too far away from the dataset, as out-of-distribution designs no longer appear promising. Iteratively regularizing $\phi$ to enforce invariance between $\Opt$ and $\mu$ until convergence will restrict the optimizer from selecting erroneously overestimated out-of-distribution inputs. We supplement this intuition with an illustration shown in Figure~\ref{fig:intuition_fig}. We will now present a formalization of this intuition in the form of a bi-level optimization problem that aims to maximize a lower-bound on the expected objective value attained under $\Opt$, and derive the training objectives for \methodname\ using this bi-level optimization formulation. 
\begin{wrapfigure}{r}{0.5\textwidth}
\vspace{-0.8cm}
\begin{minipage}{0.5\textwidth}
\begin{algorithm}[H]
\caption{\methodname: Training and Optimization}
\label{alg:iom}
\begin{small}
\begin{algorithmic}[1]
\State \textbf{Input}: training data $\data$, number of gradient steps $T=50$ to optimize $\Opt$ starting from the training distribution $\mu$, training iteration $K$, batch size $n$,
\State \textbf{Initialize}: representation model $\phi_\eta(\cdot)$, learned model $\widehat{f}_{\theta}(\cdot)$, set of optimized design particles $\{\mathbf{x}^*_i\}_{i=1}^m$
\For{$k=1\cdots, K$}, 
\State Sample $n$ training points $(\x_{i}, y_i)\sim\data$
\State Run one gradient step (Equation~\ref{eqn:training_ioms}) w.r.t $\theta_k$ and ~~~~~~~~~~~~~$\phi_k$ to obtain $\phi_{k+1}$ and $f_{\theta_{k+1}}(\phi_{k+1}(\cdot))$.
\State Run one gradient step to optimize design particles w.r.t. $f_\theta(\phi(\bx))$: $\mathbf{x}^*_i \leftarrow \mathbf{x}^*_i + \eta \nabla_\mathbf{x} f_{\theta_k}(\phi_k(\mathbf{x}^*_i))$.
\EndFor

\end{algorithmic}
\end{small}
\end{algorithm}
\end{minipage}
\vspace{-0.5cm}
\end{wrapfigure}
However, do note that, our practical method does not aim to solve a bi-level optimization exactly due to computational infeasibility.

\textbf{Bi-level optimization for invariant objective models (\methodname).} To enforce invariance at a representational level, we utilize an invariance regularizer. Denoted as $\text{disc}_{\cH}(p, q)$, our invariance regularizer is a discrepancy measure between two probability distributions $p$ and $q$, w.r.t. a certain hypothesis class $\cH$. For example, 
if $\cH$ is the class of all $1$-Lipschitz functions, $\text{disc}_\cH$ is the $1$-Wasserstein distance; 
if $\cH$ is the class of linear functions, $\text{disc}_\cH$ is the maximum mean discrepancy (MMD)~\citep{gretton2012kernel} distance. In practice, we implement this via a generative adversarial network as we discuss in the next subsection.

In addition to training ${f}_\theta$ via standard supervised regression on the labels in the training data, \methodname\ trains the representation $\phi(\bx)$ and the model $\widehat{f}(\cdot)$ to additionally minimize the discrepancy $\text{disc}_\mathcal{H}(\bbP_\mu(\phi(\bx)), \bbP_{\mu_{\Optold}}(\phi(\bx)))$
between distributions of representations under the optimized distribution, $\Opt$ and the training distribution $\mu$. This objective can be formalized as:
\begin{align} 
\label{eqn:bi_level}
	&\max_{\Opt} \hspace{0.1cm} ~~\mathbb{E}_{\widehat{\bx} \sim \Opt}\left[{f}_\theta^*(\phi^*(\widehat{\bx}))\right] \nonumber \\
	& \text{~~s.t.~~} (\phi^*, {f}^*_\theta) = \arg\min_{\phi, \widehat{f}}~~ \frac{1}{n}\sum_{i=1}^n(\widehat{f}(\phi(\bx_i)) - y_i)^2 + \lambda \cdot  \text{disc}_{\cH}(\bbP_{\mu_\text{data}}(\phi(\bx)), \bbP_{_{\mu_{\Optold}}}(\phi({\bx}))),
\end{align}
where $\lambda$ denotes a weighting hyperparameter. In Section~\ref{sec:practical_alg}, we will describe how we can convert this bi-level optimization problem into a practical MBO algorithm that can be implemented with high capacity deep neural networks. This is followed by a theoretical analysis to formally justify the objective in Section~\ref{section:invariant_representations}, and  a discussion of how the hyperparameter $\lambda$ can be tuned against a validation set in Section~\ref{sec:cross_validation}, without any online access to the ground truth function. 

\vspace{-0.25cm}
\subsection{Optimizing the Bi-Level Problem and the Practical \methodname\ Algorithm}
\label{sec:practical_alg}
\vspace{-0.25cm}

In this section, we will describe how we can convert the abstract bi-level optimization problem above into an objective that can be trained practically and discuss some practical implementation details. \methodname\ alternates between solving the bi-level optimization problem with respect to the distribution $\Opt$ and the learned model $f_\theta$ and $\phi$ independently.
For training the learned model and the representation $f_\theta$ and $\phi$, \methodname\ solves the inner optimization by running one-step of gradient descent ($\mathsf{GD}$), utilizing the current snapshot of $\Opt$ (denoted as $\Opt^t$):
\begin{equation} 
\label{eqn:training_ioms}
	(\phi_{t+1}, \theta_{t+1}) = \mathsf{GD}_{\theta, \phi; t}\left( \frac{1}{n}\sum_{i=1}^n({f}_\theta(\phi(\bx_i)) - y_i)^2 + \lambda \cdot  \text{disc}_{\cH}(\bbP_{\mu_\text{data}}(\phi(\bx)), \bbP_{_{\mu^t_{\Optold}}}(\phi({\bx}))) \right). 
\end{equation}
Simultaneously, we update the design particles representing $\Opt^t := \delta\{\bx^1_t, \bx^2_t, \cdots\}$ towards optimizing the learned function $f_{\theta_{t+1}}(\phi_{t+1}(\bx))$:
$\bx^i_{t+1} \leftarrow \bx^i_t + \eta \nabla_\bx f_{\theta_{t+1}}(\phi_{t+1}(\bx))|_{\bx = \bx^i_t}, ~\forall i.$ Pseudocode for \methodname\ is shown in Algorithm~\ref{alg:iom}. 

\textbf{Practical implementation details:} We model the representation $\phi(\bx)$ and the learned function ${f}_\theta(\cdot)$ each as two-hidden layer ReLU networks with sizes 2048 and 1024, respectively. We utilize the $\chi^2$-discrepancy, $\text{disc}_\mathcal{H}$ in our experiments: $\text{disc}_\phi(p, q) := \chi^2([\phi(\bx)]_{\bx \sim p}, [\phi(\bx)]_{\bx \sim q})$.
This divergence can be instantiated in its variational form via a least-squares generative adversarial network (LS-GAN)~\citep{mao2017least} on the representations $\phi(\bx)$. Concretely, we train an LS-GAN discriminator to discriminate between the representations $\phi(\bx)$ under the training distribution $\mu(\bx)$ and the optimized distribution $\Opt(\bx)$, whereas the feature representations under these distributions are optimized to be as invariant to each other as possible. 
We also evaluated a variety of discrepancy measures, including the maximum mean discrepancy (MMD)~\citep{gretton2012kernel}, but found the $\chi^2$-divergence to be the best performing version across the board. More implementation details can be found in Appendix~\ref{appsec: exp}.

\vspace{-0.23cm}
\subsection{Theoretical Analysis of Invariant Representation Learning for MBO}
\label{section:invariant_representations}
\vspace{-0.23cm}
In this section, we will formally show that under certain standard assumptions, solving the bi-level optimization problem (Equation~\ref{eqn:bi_level}) enables us to find a better design than the dataset distribution $\mu_\text{data}$. To show this, we will combine theoretical tools for analyzing offline RL algorithms~\citep{cheng2022adversarially,kumar2020conservative,yu2020mopo} and domain adaptation~\citep{zhao2019learning}. We show that the expected value of the ground truth objective under any given distribution $\pi(\bx)$ of design inputs can be lower bounded in terms of the expected value under the learned function modulo some error terms. Then we will show that these error terms are controlled tightly when solving the bi-level problem, and finally appeal
to uniform concentration to show that the performance of $\Opt$ is better with high probability. Our goal isn't to derive the tightest possible bound for \methodname, but show that \methodname\ can attain reasonable performance guarantees, and this domain adaptation perspective can tackle distributional shift.

\textbf{Notation and assumptions.} Following standard assumptions~\citep{lee2021representation}, we will analyze the setting when the learned representation lies in some space of functions, $\phi \in \Phi$, and the learned function $f_\theta$ lies in $f_\theta \in \mathcal{F}$. We will consider that the optimized distribution $\Opt$ belongs to a function class, $\Opt \in \Pi$. Our analysis will require that the ground truth model $f$ is approximately realizable under some representation $\phi \in \Phi$ and for some $g \in \mathcal{F}$. That is, we assume:
\begin{assumption}[$\varepsilon-$realizability]
\label{assumption:realizability}
For any distribution $\pi$ over designs, there exists a representation $\phi \in \Phi$ and a $g \in \mathcal{F}$ such that $\min_{g} \max_{\pi \in \Pi} \mathbb{E}_{\bx \sim \pi}\left[|f(\bx) - g(\phi(\bx))|\right] \leq \varepsilon_{\mathcal{F}, \Phi}$.   
\end{assumption}
\vspace{-0.1cm}
In addition, we assume that $\forall f \in \mathcal{F}, ||f||_\infty < \infty$ and is Lipschitz continuous with constant $C$: $||f||_\mathrm{L} \leq C$. This is true for most parameteric function classes, such as neural networks. Then, we can lower bound the average objective value under any distribution in terms of the invariance:   
\begin{proposition}[(Informal) Lower-bounding the ground truth value under $\pi$]
\label{prop:representation_error}
Under Assumption~\ref{assumption:realizability} and the Lipschitz continuity on the function $f_{\theta}$, the ground truth objective for any given $\pi \in \Pi$ can be lower bounded in terms of the learned objective model $f_\theta \in \mathcal{F}$ and representation $\phi \in \Phi$, with high probability $\geq 1 - \delta$ that:
\begin{align}
\label{eq:plugging_back_in}
     J(\pi) -& {J}_\theta(\pi) \geq \underbrace{{J}(\mu_\text{data}) - {J}_\theta(\mu_\text{data})}_{(\blacksquare)}-\underbrace{C_\mathcal{F} \cdot \text{disc}_{\cH}(\bbP_{\mu_\text{data}}(\phi(\bx)), \bbP_{\pi}(\phi(\bx)))}_{(*)} - \varepsilon_{\mathcal{F}, \Phi} - \varepsilon_{\text{stat}}, \nonumber
\end{align}
where $C_{\mathcal{F}}$ is a uniform constant depending on the function class $\mathcal{F}$ and $\varepsilon_\text{stat}$ refers to statistical error that decays inversely with the dataset size, $|\mathcal{D}|$.  
\end{proposition}
A proof for Proposition~\ref{prop:representation_error} is in Appendix~\ref{sec:proof}.
The $\blacksquare$ term can further be lower bounded using a standard concentration error bound for ERM, and can be controlled if we train ${f}_\theta(\phi(\cdot))$ minimize prediction error against ground truth values of $f$. Now we utilize Proposition~\ref{prop:representation_error} and a uniform concentration argument to lower-bound the value under the optimized distribution $\Opt$.  

\begin{proposition}[(Informal) Performance guarantee for \methodname] 
\label{proposition:improvement}
Under Assumption~\ref{assumption:realizability}, the expected value of the ground truth objective under $\Opt$, $J(\Opt)$ is lower bounded by:
\begin{align}
    J(\Opt) \gtrsim J(\mu_\text{data})  - \mathcal{O}\left(\sqrt{\frac{\log \frac{|\mathcal{F}| |{\Phi}| |{\Pi}|}{\delta}}{|\mathcal{D}|}} + \frac{\log \frac{|\mathcal{F}| |{\Phi}| |{\Pi}|}{\delta}}{|\mathcal{D}|} \right) + \underbrace{{J}_\theta(\Opt) - {J}_\theta(\mu_\text{data})}_{(\circ)} - (*). 
\end{align}
\end{proposition}
\vspace{-0.1cm}

A proof of Proposition~\ref{proposition:improvement} can be found in Appendix~\ref{sec:proof}. This proposition implies that as long as the optimizer finds designs that improve the learned objective function $f_\theta \in \mathcal{F}$, such that $(\circ)$ is positive, and the discrepancy term $(\star)$ is minimized, then the distribution over designs found by \methodname, $\Opt$, will compete favorably against the distributions of designs in the dataset. Of course, these bounds present the worst-case guarantees for \methodname, but as we  will show in our experiments, \methodname\ does clearly improve over the best designs in the offline dataset.  
In addition, this theoretical analysis also motivates the design of our scheme for performing hyperparameter tuning and model selection. This scheme will be based directly on the result from Proposition~\ref{prop:representation_error}. We will discuss this connection when discussing tuning in Section~\ref{sec:cross_validation}.

Finally, we would like to highlight how invariance leads to ``average'' values if $\Opt$ deviates too far away from the data distribution $\mu$. If $\Opt$ deviates too far, such that $\text{disc}_{\cH}(\bbP_{\mu_\text{data}}(\phi), \bbP_{\Opt}(\phi)) \leq \varepsilon'$, the learned function is bounded too close to the average learned value on the data i.e., $|\widehat{J}(\Opt) - \widehat{J}(\mu_\text{data})| \leq C' \cdot \varepsilon'$. Thus, the learned function simply reverts to ``predicting the mean objective value'' of the training distribution as $\varepsilon'$ decreases,
thereby losing any information about the identity of the optimized distribution, when the optimizer goes too far out of $\mu$.

\vspace{-0.2cm}
\section{Offline Workflow and Tuning for \methodname}
\label{sec:cross_validation}
\vspace{-0.2cm}

A central challenge in offline model-based optimization is that of offline workflow and hyperparameter tuning: not only must an algorithm be able to produce designs using a static dataset, but it must also prescribe a \emph{fully offline} workflow for tuning hyperparameters so that the produced designs are as good as possible. This includes both explicit hyperparameters (such as the coefficient of conservatism or invariance, $\lambda$ for \methodname) and implicit hyperparameters (such as the early stopping point). The primary challenge stems from the fact that unlike supervised learning, the final performance in MBO is determined by the optimized distribution $\Opt$, which is different from the training distribution, $\mu_\text{data}$, and we must reason about this performance fully offline.

In this section, we discuss how \methodname\ admits a particularly convenient and \emph{fully offline} hyperparameter tuning procedure that works well empirically. The goal of our tuning problem is to find a tuple consisting of the learned representation $\phi^\lambda$, the learned model, $f^\lambda_\theta$, and the optimized distribution $\Opt^\lambda$, out of a given set of tuples obtained for different values of $\lambda$. In other words, we wish to find the $\lambda$ from the set of models $M_{\lambda}:=({f}^{\lambda}_\theta, {\phi}^{\lambda}, \mu^{\lambda}_{\Opt})$, such that the corresponding $\Opt^\lambda$ attains the highest value under the ground truth objective. Crucially note that this must be done fully offline.

\textbf{Our key idea:} The primary question we wish to answer is: how can we identify if a given tuple $(f^\lambda_\theta, \phi^\lambda, \Opt^\lambda)$ leads to good performance? In order for a tuple trained via \methodname\ to perform well, we need two primary conditions: \textbf{(1)} it should be nearly invariant, and as a result, robust to distributional shift, and \textbf{(2)} the learned $f^\lambda_\theta$ must attain good generalization performance within the training distribution. Intuitively \textbf{(1)} guards us against distributional shift, and \textbf{(2)} enables us to find as good of a learned model on the dataset, that does not overfit or underfit as a result of the invariance regularizer. This forms the basis of our tuning method.

\textbf{Offline tuning of \methodname:} To instantiate the idea practically, we first run \methodname\ with a set of hyperparameters $\Lambda = \{\lambda_1, \lambda_2, \cdots\}$. Then for each run, we record the validation in-distribution error and the value of the invariance regularizer on a validation set. We then perform tuning in two steps: first, we filter all $\lambda$ values that attain near-perfect invariance under a user-specified margin $\varepsilon$, i.e., $\text{disc}_{\cH}(\bbP_{\mu_\text{data}}({\phi}^{\lambda}(\bx)), \bbP_{{\mu^{\omega}_{\Optold}}}({\phi}^{\lambda}({\bx}))) \leq \epsilon$, and are hence guaranteed to be robust to distributional shift. Second, we now pick models that attain good performance within the training distribution by selecting $\lambda$ values that attain the smallest validation prediction error: $(f_\theta(\phi(\bx)) - y)^2$, in addition to picking the early stopping point based on the smallest validation error. This enables us to pick $\lambda$ that gives rise to a generalizing model $f_\theta^\lambda, \phi^\lambda$, while being most robust to distributional shift. This still leaves a user-defined margin $\varepsilon$ as a free parameter, but we find that it is comparatively insensitive and can be selected heuristically, as we show in our experiments in Section~\ref{sec:ablations}.

Formally, this procedure corresponds to solving the following optimization:
\begin{align}
\label{eqn:tuning_strategy}
    \lambda^* : =  \arg\min_{\lambda\in\Lambda}~~ \frac{1}{n}\sum_{i=1}^{n_{\text{val}}}({f}_\theta^{\lambda}({\phi}^{\lambda}(\bx_i)) - y_i)^2 \hspace{0.2cm}
    \text{s.t.} \hspace{0.2cm} \text{disc}_{\cH}(\bbP_{\mu_{\text{val}}}({\phi}^{\lambda}(\bx)), \bbP_{{\mu^{\lambda}_{\Optold}}}({\phi}^{\lambda}({\bx}))) \leq \epsilon
\end{align}
Finally, for a more formal explanation of the soundness of this strategy, we note that applying Equation~\ref{eqn:tuning_strategy} provides a tight lower bound on the ground truth objective value in Proposition~\ref{prop:representation_error}. Up to statistical error (which decays as $|\mathcal{D}|$ increases), the prediction error on a held-out validation set estimates $(\blacksquare)$ in Proposition~\ref{prop:representation_error}, and a smaller discrepancy (i.e., $\mathrm{disc}_\mathcal{H} \leq \varepsilon$) controls the $(\star)$ term in Prop.~\ref{prop:representation_error}. $\varepsilon_\text{stat}$ is statistical error, and $\varepsilon_{\mathcal{F}, \Phi}$ only depends on the function classes involved ($\mathcal{F}, \Phi$). This tuning $\lambda$ per Equation~\ref{eqn:tuning_strategy} can be justified as aiming to maximize a lower bound on $J(\Opt)$.  

\textbf{Practical implementation.} When implementing this procedure in practice, we must first select the value of $\varepsilon$. For obtaining perfect invariance, the value of $\varepsilon$ must be equal to 0.25 (when the discriminator is trained using a $\chi^2$-divergence least-squares GAN objective), and therefore, $\varepsilon$ must be set close to 0.25. However, in practice, instead of selecting $\varepsilon$, we can select a fraction of all possible $\lambda$ values that pass the filter. In our experiments, we sort the runs based on the absolute difference between the value of invariance regularizer (i.e., the discriminator training objective) and 0.5, and let the top 15\% (top 3 out of the 7 $\lambda$ values we tune on) of the runs be selected for the second step.     

In Section~\ref{sec:exp}, we validate this offline tuning strategy, both against other alternative offline strategies for tuning \methodname\ as well as offline strategies for tuning prior methods (COMs~\citep{trabucco2021conservative} and gradient descent~\citep{trabucco2021designbench}). We find that our offline tuning strategy performs comparably to full online tuning without needing any online queries, whereas other strategies for \methodname\ and other methods, lead to significantly poor performance compared to online tuning.

\vspace{-0.2cm}
\section{Related Work}
\vspace{-0.2cm}
Model-based optimization (MBO) or black-box optimization~\citep{brookes19a,Angermueller2020Model-based,snoek15scalable,snoek2012practical,mirhoseini2020chip,yu2021roma} has been studied widely in prior works. Among these works, the central challenge is in handling the distributional shift that arises between the distribution of designs seen in the dataset and those taken by the optimizer~\citep{brookes19a, fannjiang2020autofocused, kumar2019model}. For methods that learn a model, this typically manifests itself as \emph{model exploitation}, where the optimizer can find out-of-distribution designs for which the model erroneously predicts good values, analogously to adversarial examples. This distributional shift is distinct from what is studied in prior works that tackle misspecification in GPs and bandits~\citep{neiswanger2021uncertainty,nguyen2021value,kirschner2021bias} or methods that tackle robustness across contexts~\citep{kirschner2020distributionally} in the online setting. 

To address distribution shift, several recent works have proposed to be conservative. This is majorly the case in other decision-making problems such as offline reinforcement learning (RL)~\citep{levine2020offline} and contextual bandits~\citep{nguyen2021offline}, where instead of optimizing for a design, we must find the optimal action at each state. In offline MBO, prior works have tried to be conservative in two different ways: either by constraining the optimized designs to lie close to the training designs~\citep{brookes19a,fannjiang2020autofocused,kumar2019model,peng2019advantage,kumar2019stabilizing} using generative models~\citep{kingma2013autoencoding,goodfellow2014generative}, or by learning a conservative ~\citep{trabucco2021conservative,kumar2021data,kumar2020conservative} or calibrated~\citep{Fu21nemo,yu2021roma} learned model. While the latter class of methods has worked well in practice~\citep{trabucco2021designbench,kumar2021data}, these methods can sometimes be overly conservative and are quite sensitive to hyperparameters. Sensitivity to hyperparameters is acceptable as long as one can attain good performance using reasonable offline tuning strategies, without accessing the groundtruth function, but we find in our experiments (Section~\ref{sec:ablations}) this is not easy for conservative methods. On the contrary, \methodname\ not only does not require a generative model over the design space, which can be intractable with high-dimensions (e.g., 5000-d designs in HopperController), but also attains better performance and provides a convenient tuning rule.

In this work we formulate data-driven model-based optimization as domain adaptation~\citep{ben2010theory,ganin2015unsupervised}, and derive algorithms for model-based optimization from this perspective. This perspective has been used to devise decision-making methods for treatment effect estimation~\citep{johansson2016learning, johansson2018learning, yao2018representation} and off-policy evaluation in RL~\citep{liu2018representation, lee2021representation}. While these prior works do use an invariance regularizer, they primarily focus on obtaining more accurate estimations, whereas our goal is to utilize invariance to combat distributional shift during optimization. Thus, our usage of invariant representations is tailored towards the optimizer, instead of aiming for accurate predictions everywhere like these prior works, which is impossible in general. Furthermore, our analysis uses invariant representations to derive procedures for tuning and early stopping, while these prior works do not explore these perspectives.

\vspace{-0.3cm}
\section{Experimental Evaluation}
\label{sec:exp}
\vspace{-0.3cm}
The goal of our empirical evaluation is to answer the following questions: \textbf{(1)} Does \methodname\ outperform prior state-of-the-art methods for data-driven model-based optimization? \textbf{(2)} Is our offline tuning strategy for \methodname\ better than other offline tuning strategies for \methodname\ or other prior methods? \textbf{(3)} How sensitive is \methodname\ to the choice of the hyperparameter $\lambda$? In this section, we will answer these questions via a comparative evaluation of \methodname\ on several continuous data-driven design tasks from the standard tasks from the Design-Bench suite~\citep{trabucco2021designbench}. 

\vspace{-0.3cm}
\subsection{Empirical Evaluation on Benchmark Tasks}
\label{sec:empirical_evals}
\vspace{-0.3cm}
We compare \methodname\ to a variety of prior methods for model-based optimization: CbAS~\citep{brookes19a}, Autofocused CbAS~\citep{fannjiang2020autofocused} and MINs~\citep{kumar2019model} that utilize generative models for constraining distributional shift; COMs~\citep{trabucco2021conservative} and RoMA~\citep{yu2021roma} train robust models  of the objective function by via conservatism and smoothness, respectively; gradient-free optimization methods such as REINFORCE~\citep{Williams92}, CMA-ES~\citep{Hansen06} and conventional Bayesian optimization BO-qEI~\citep{reparameterization2017}. We also compare the standard approach of learning a na\"ive model of the objective function, which is then optimized via gradient ascent. To better understand the benefits of enforcing invariance, we also compare \methodname\ to \methodname-C, which additionally applies a conservatism term on top of invariance that pushes up the learned values of designs in the dataset. Additionally, we add a prior method that optimizes designs against a lower-confidence bound estimate computed using a Gaussian process posterior in Appendix~\ref{subsec:app_result}.

\textbf{Tasks.} We evaluate on four tasks with continuous-valued input space from the Design-Bench~\citep{trabucco2021designbench} benchmark for offline model-based optimization (MBO). Here we briefly summarize these tasks:  \textbf{(1)} Superconductor: this task requires optimizing the properties of superconducting materials to maximize the operating temperature; the dataset consists of 17010 training points, and each input $\bx\in\bbR^{86}$ represents a serialization of the chemical formula. \textbf{(2)} D'Kitty and \textbf{(3)} Ant robot morphology design: these tasks require selecting the optimal morphology of a quadrupedal robot, in order to maximize how quickly it can walk or crawl in a particular direction; the morphologies are represented by $\bx\in\bbR^{56}$ and $\bx\in\bbR^{60}$, respectively. \textbf{(4)} Hopper controller: this task reframes the classic hopper task~\citep{gym} as a bandit problem, where instead of sequentially controlling the hopper, the goal is to directly output the parameters of a neural network controller that maximizes hopping speed; note that this task, as set up in Design-Bench, is \emph{not} a RL problem, but rather a bandit problem with a 5126-dimensional continuous action space corresponding to neural network weights.

\textbf{Evaluation protocol.} Following prior work~\citep{trabucco2021designbench, trabucco2021conservative, yu2021roma,Fu21nemo}, for each method, we query the learned model to obtain the top performing $N$ points, i.e., $\bx_1, \bx_2, \cdots, \bx_N$. We report the maximum objective value among these $N$ samples. This protocol reasonably reflects a real-world design process, where a number of computationally produced designs are tested and the best one is used for deployment.

\textbf{Hyperparameter tuning.} For fair comparison against prior methods studied in \citet{trabucco2021designbench}, we follow the uniform tuning strategy for reporting results in this section: we run \methodname\ on each task for a given set of $\lambda$ values, and pick a single value of $\lambda$ across every task. Note that this does require access to online evaluation, but prior works do use the same protocol. We will evaluate the efficacy of our offline tuning approach in Section~\ref{sec:ablations} extensively.  

\begin{figure}[htp]
\vspace{-0.35cm}
    \centering
    \includegraphics[height=3.2cm]{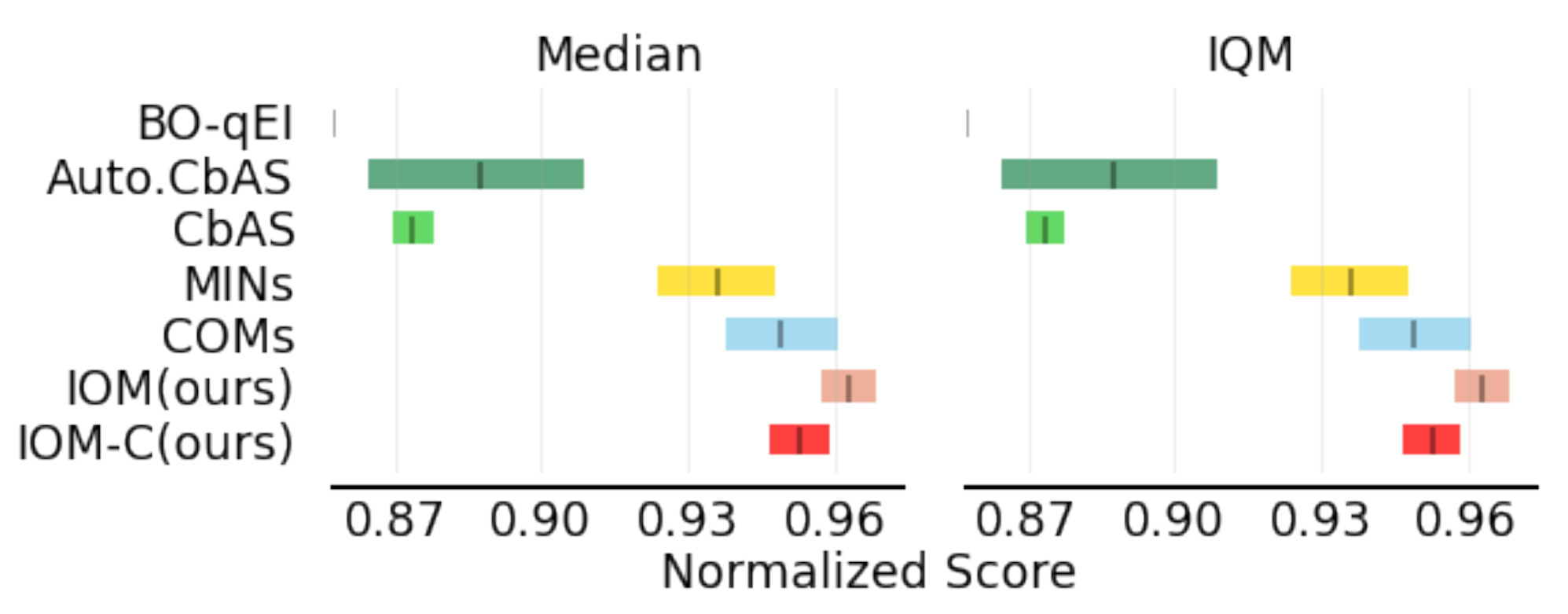}
    ~\vline~
    \includegraphics[height=3.5cm]{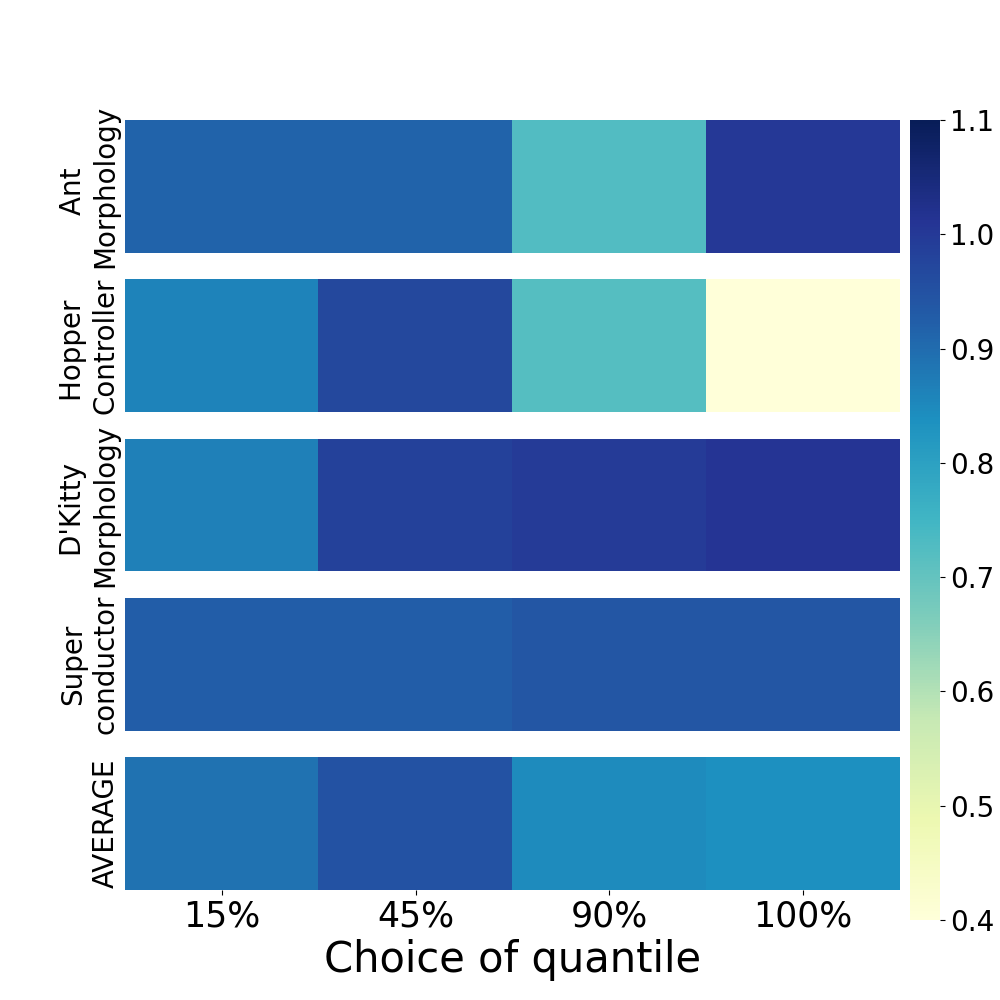}
    \vspace{-0.2cm}

    \caption{\footnotesize \textbf{Left}: Median and IQM~\citep{agarwal2021deep} (with 95\% Stratified Bootstrap CIs) for the aggregated normalized score on Design-bench. \methodname\ improves over the prior methods, including those which use explicit conservatism although \methodname\ does not. As we could not find the individual runs for RoMA~\citep{yu2021roma}, we report the mean and standard deviation for RoMA (copied from the results in~\citet{yu2021roma}), along with more detailed results for other baselines in Appendix~\ref{subsec:app_result}. \textbf{Right}: The offline tuning performance matrix under different quantiles to filter the models based on discriminator loss, compared with the oracle uniform tuning, which shows our tuning strategy is robust to this hyper-parameter.}
    \label{fig:perf}
    \vspace{-0.35cm}
\end{figure}
\textbf{Results.} Following the convention proposed by \citet{trabucco2021designbench}, we evaluate all methods in terms of normalized score (higher is better). In Figure~\ref{fig:perf} (left), we show the median and interquantile mean (IQM)~\citep{agarwal2021deep} for the aggregated scores across all tasks. Due to space limits, we only report the scores for the top-performing baselines (See Appendix~\ref{appsec: exp} for the results for all baselines). \methodname\ significantly outperforms all of the prior methods, including state-of-the-art methods that employ conservatism, such as COMs~\citep{trabucco2021conservative}. This indicates that invariance alone (\methodname) can effectively handle distribution shifts and avoid over-estimation for out-of-distribution actions. Additionally, a comparison between the performance of \methodname\ and \methodname\ with conservatism (\methodname-C) suggests that adding conservatism does not improve performance further. This further corroborates the fact that invariance is effective in preventing the optimizer from going out of distribution and conservatism is not needed. That said, we would like to clarify that this result does not mean that invariance will always be better than conservatism on any given offline MBO problem. In order to understand why \methodname\ outperforms COMs on these tasks, we examine the in-distribution validation error of \methodname\ and COMs in Appendix~\ref{app: additional} and find that \methodname\ learns less distorted functions that generalize better.

\vspace{-0.27cm}
\subsection{Analysis of the Offline Workflow for Tuning \methodname} 
\label{sec:ablations}
\vspace{-0.25cm}
We will now present experiments to understand the effectiveness of our offline tuning scheme (Section~\ref{sec:cross_validation}) for tuning \methodname. 

It is important to highlight that, in general, offline tuning presents significant challenges for MBO. Prior works in bandits and reinforcement learning~\citep{zanette2021exponential} have suggested that offline tuning should not work at all and, to the best of our knowledge, no prior work in offline MBO has proposed a principled and effective offline tuning method. In this section, we empirically validate our tuning strategy in comparison to other potential strategies, and compare results from our  approach to an oracle strategy that uses online evaluation of the true function.

\textbf{Comparisons.} We compare our offline tuning strategy to various offline strategies: (1). only tunes \methodname\ based on validation-set prediction error, analogous to supervised learning (called ``prediction error only''); (2). a strategy that tunes \methodname\ based on only the value of the invariance regularizer (called ``invariance only''); (3). a method that utilizes our tuning strategy for picking $\lambda$, but not for picking the checkpoint (``uniform checkpoint selection'') and two offline tuning strategies for tuning other prior methods: (a) tuning COMs~\citep{trabucco2021conservative} using validation-set prediction error (called ``prediction error only''), conservatism loss (called ''conservatism loss only'') and both prediction error and conservatism loss, one by one (called ``two-step tuning''); and (b) tuning an objective model trained via gradient ascent~\citep{trabucco2021designbench} using validation MSE error (``gradient ascent''). We report the drop in performance relative to the corresponding oracle online strategy that actually computes the ``groundtruth'' values of optimized designs in the simulator to choose hyperparameters in hindsight, i.e., the relative drop in performance of the method (\methodname, COM, grad. ascent) when tuned offline using the particular strategy compared to oracle tuning for the same method.

\textbf{Results.} We present our results in Figure~\ref{fig:ablation} (right) comparing the efficacy of various tuning schemes across various methods averaged over all tasks. While offline tuning does lead to a reduction in performance compared to oracle online tuning, we find that the drop is smallest for our offline tuning strategy applied to \methodname. All other tuning strategies based on either only prediction error or the invariance regularizer lead to worse performance. This indicates that our tuning strategy is effective in tuning \methodname. Furthermore, offline tuning strategies that utilize a validation set with prior methods, such as COMs and gradient ascent, also lead to a significant decrease in performance relative to the oracle. This is quite appealing as it suggests that \methodname\ is more amenable to offline tuning strategies relative to prior offline MBO algorithms. Even if these prior methods can perform well when provided access to online evaluations, they might start to fail in real-world MBO problems when one needs to tune them fully offline. We discuss computational complexity of tuning the methods in Appendix~\ref{app: additional}.

\textbf{Ablation analyses.} Next we aim to study the sensitivity of our offline tuning strategy with respect to the non-rigorous, user-specified parameter $\varepsilon$ in Equation~\ref{eqn:tuning_strategy} and checkpoint selection based on the validation prediction error. Figure~\ref{fig:perf} (right) shows the offline tuning performance heatmap for all tasks (along their aggregated performance in the bottom row) under different choice of quantiles (where quantile $x\%$ means the top $x\%$ of the models according to the value of the invariance regularizer, and the given value of $\epsilon$ were chosen). The color in each position of the matrix corresponds to the ratio of performance of our offline tuning strategy for a given quantile value, compared to that of uniform oracle tuning. As shown in the bottom row, aggregated over all tasks, we can see our offline selection strategy is quite robust to the hyperparameter $\epsilon$. Figure~\ref{fig:ablation} (left) evaluates the efficacy of our checkpoint selection strategy, we plot the objective value $J(\Opt)$ attained as a function of the number of training epochs (i.e., checkpoints), alongside the values for the prediction MSE in the validation set for the corresponding epochs. The checkpoints selected by our early stopping scheme are shown via the vertical line. Observe that indeed the checkpoint for the lowest validation MSE corresponds to a rather good optimized distribution $\Opt$ for different tasks, no matter the optimal training epochs happen at the beginning or near the end, which validates our early stopping strategy. More detailed results showing the efficacy of early stopping for both \methodname\ can be found in Appendix~\ref{subsec:app_result}.

\begin{figure}[htp]
\vspace{-0.3cm}
    \centering

    \includegraphics[width=0.57\linewidth]{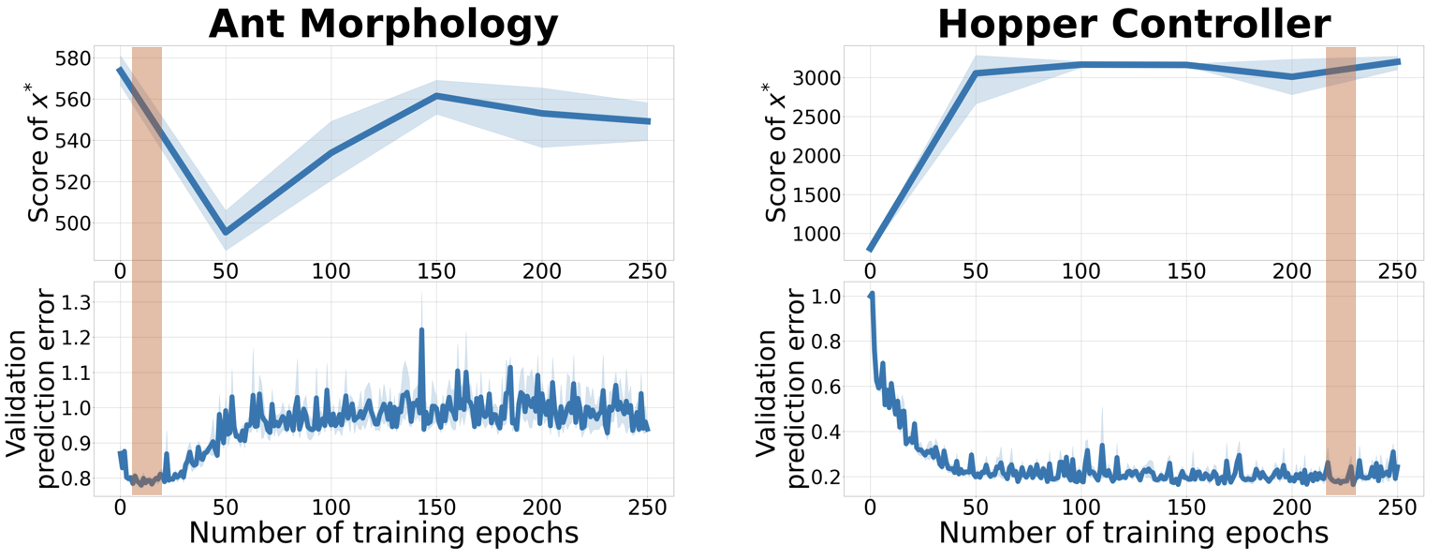}
    ~\vline~
    \includegraphics[width=0.37\linewidth]{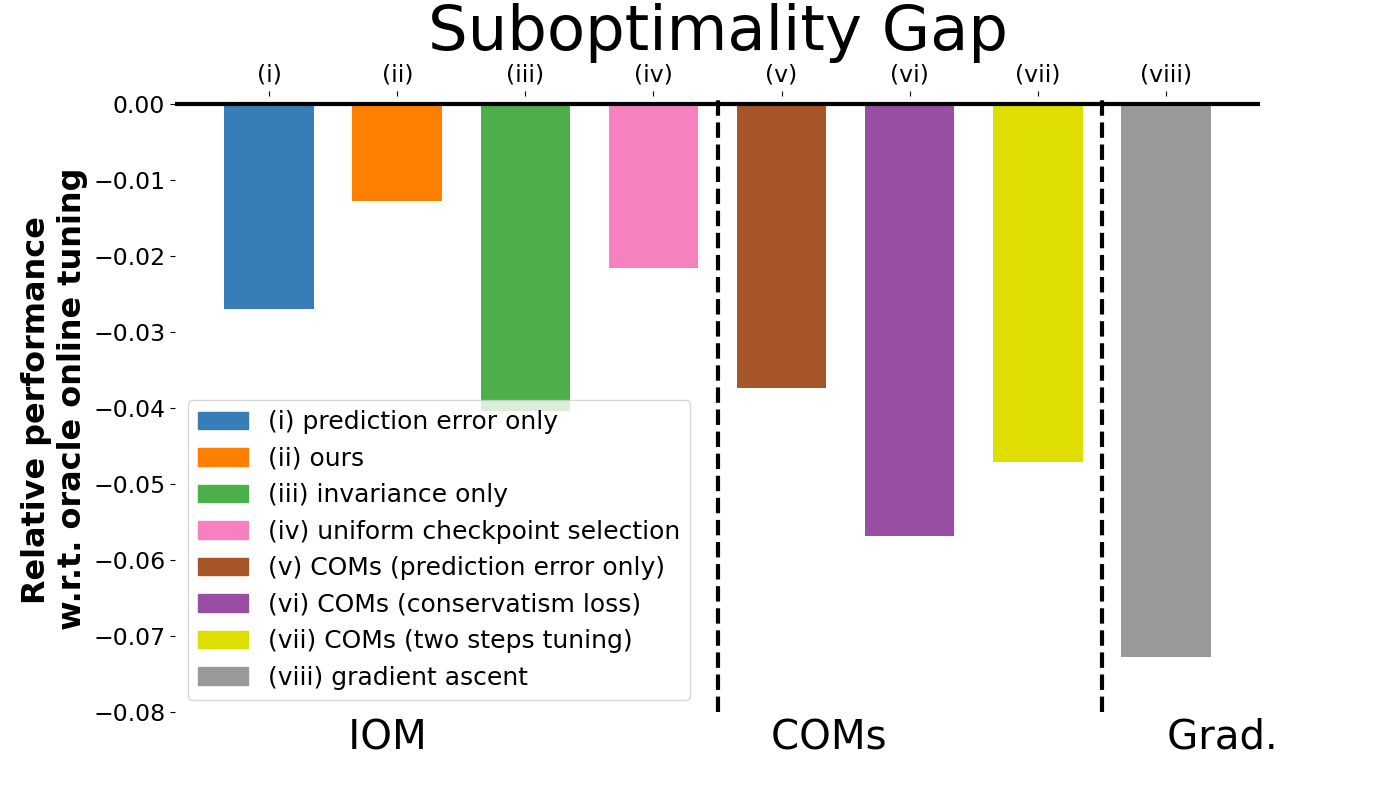}
    \vspace{-0.2cm}
     \caption{\footnotesize \textbf{Left}: Visualizing our checkpoint selection scheme on two tasks, and no matter the best checkpoint happens at the beginning or end of the training, our strategy based on validation in-distribution MSE could capture the best performance. \textbf{Right}: Offline tuning vs oracle uniform online tuning for \methodname, COMs and Gradient-ascent. Among all the methods, \methodname\ achieves the smallest regret for offline tuning, and our specific tuning strategy works best compared with other tuning strategies for \methodname.}
    \label{fig:ablation}
\end{figure}

\vspace{-0.4cm}
\section{Discussion}
\label{sec:discussion}
\vspace{-0.3cm}
In this work we study offline data-driven model based optimization (MBO), where the goal is to produce optimized designs, purely using a static dataset. We proposed to view offline MBO through the lens of domain adaptation, and proposed a method that trains a model of the objective with an additional regularizer to promote invariance between the representations learned by the model in expectation under the training data distribution and the distribution of optimized designs. We show that this leads to a practically effective data-driven optimization method, and an appealing workflow for selecting the hyperparameters with offline data. Our method outperforms prior approaches, which illustrates the practical benefits of invariance. We also validate the effectiveness of the accompanying tuning strategy. While we evaluate our method on MBO problems, our approach can in principle also be extended to contextual bandits and even to, sequential reinforcement learning. But we admit that there might be some limitations: \textbf{(1)} fully data-driven offline optimization may on its own be insufficient (due to limitations of staying close to the dataset, hardness of tuning) and inevitably will require access to online evaluation. \textbf{(2)} a full understanding of theoretical conditions when invariance can outperform conservatism is important, but out of the scope of this work.

\vspace{-0.4cm}
\section*{Acknowledgements}
\vspace{-0.35cm}

We thank Brandon Trabucco and Xinyang Geng for help in setting up the Design bench and COMs codebases. We thank anonymous NeurIPS reviewers, members of the RAIL lab at UC Berkeley and Amy Lu for informative discussions, suggestions and feedback on an early version of this paper. This research was supported by the Office of Naval Research, C3.ai, and Schmidt Futures, and compute resources from Google cloud. AK is supported by the Apple Scholars in AI/ML PhD fellowship.

\bibliographystyle{plainnat}
\bibliography{neurips_2022}

\newpage

\appendix

\part*{Appendices}

\section{Proofs of Theoretical Results}
\label{sec:proof}

In this section, we will provide proofs of the various theoretical results from Section~\ref{section:invariant_representations}. We will first discuss the assumptions and technical conditions we use, then prove the lower bound result from Proposition~\ref{prop:representation_error} and finally utilize it to prove the performance guarantee for invariant representation learning (Proposition~\ref{proposition:improvement}). Before diving into the proofs, we first recall the set of assumptions and technical conditions below:

\subsection{Assumptions, Technical Conditions and Lemmas}
Recall that our proofs operate under the following assumptions:

\begin{itemize}
    \item \textbf{$\varepsilon$-realizability:} For any distribution $\pi$ over designs, there exists a representation $\phi \in \Phi$ and a $g \in \mathcal{F}$ such that $\min_{\phi, g} \max_{\pi \in \Pi} \mathbb{E}_{\bx \sim \pi}\left[|f(\bx) - g(\phi(\bx))|\right] \leq \varepsilon_{\mathcal{F}, \Phi}$.   
    \item $\forall f \in \mathcal{F}, ||f||_\infty < \infty$.
    \item $\forall f \in \mathcal{F}$, $||f||_\mathrm{L} \leq C$.
\end{itemize}

Additionally, we will also define statistical error that will appear in the bounds:
\begin{align}
\label{eqn:statistical_error}
    \varepsilon_\text{stat} \approx \sqrt{\frac{\log \left(\frac{|\mathcal{F}| |\Phi| c_0}{\delta}\right)}{|\mathcal{D}|}}
\end{align}
We will use some helpful Lemmas that we list below, which can be proven using standard results from empirical risk minimization and concentration inequalities~\citep{wainwright2019high}.

\begin{lemma}[Concentration of expected function values]
\label{lemma:standard_conc}
For any given $f \in \mathcal{F}$ satisfying the assumptions above, and any given distribution $\pi$ over designs $\mathcal{X}$, let $J_f(\pi) := \mathbb{E}_{\bx \sim \pi}[f(\bx)]$. Then, given a set of $N$ samples, $\bx_1, \bx_2, \cdots, \bx_N \sim \pi$, the empirical mean $\widehat{J}_f(\pi) = \frac{1}{N} \sum_{i=1}^N f(\bx_i)$ converges to $J_f(\pi)$ such that, with high probability $\geq 1 - \delta$ ($c_0$ is a universal constant),
\begin{align}
    \left\vert J_f(\pi) - \widehat{J}_f(\pi) \right\vert \lesssim \sqrt{\frac{\log \left(\frac{c_0}{\delta}\right)}{N}}.
\end{align}
\end{lemma}

Building on Lemma~\ref{lemma:standard_conc}, we now bound the empirical risk and the population risk of a given loss function $\mathcal{L}(f)$ at the empirical risk minimizer using a standard uniform concentration argument:

\begin{lemma}[Concentration of empirical risk minimizer]
\label{lemma:uniform_conc}
Given any bounded, Lipschitz loss function, $\mathcal{L}$ and for a given $f \in \mathcal{F}$, let, $\mathcal{L}(f) := \mathbb{E}_{\bx \sim \mu_\text{data}}[L(f, \bx)]$, and let the corresponding empirical loss function be denoted by: $\widehat{\mathcal{L}}(f) := \frac{1}{|\mathcal{D}|} \sum_i L(f, \bx_i)$. Then, with high probability $\geq 1 - \delta$, 
\begin{align}
    \left\vert \mathcal{L} \left(\arg \inf_{f \in \mathcal{F}} \widehat{\mathcal{L}}(f) \right) - \mathcal{L} \left(\arg \inf_{f \in \mathcal{F}} \mathcal{L}(f) \right) \right\vert \lesssim \mathcal{O}\left( ||\mathcal{L}||_\infty \sqrt{\frac{1}{|\mathcal{D}|} \log \left( \frac{|\mathcal{F}| c_0}{\delta} \right)} \right).
\end{align}
\end{lemma}

Finally, we discuss a technical condition that will be required to prove Proposition~\ref{prop:representation_error}. We consider the following ``continuity'' property of the function class $\mathcal{F}$ with respect to a loss function that is a mixture of two loss functions:
\begin{definition}[Continuity of $\mathcal{F}$ with respect to $\mathcal{L}_\lambda$.]
\label{def:continuity}
We say that a function class $\mathcal{F}$ is continuous with respect to a given loss function $\mathcal{L}_\lambda(f) := \mathcal{L}_1(f) + \lambda \mathcal{L}_2(f)$, where $\lambda \in \mathbb{R}^+$ is a hyperparameter, when the following holds: denoting $f_\theta^\lambda := \inf_{f \in \mathcal{F}} \mathcal{L}_\lambda(f)$, the $\lim_{\lambda \rightarrow 0} f_\theta^\lambda$ exists, and corresponds to a minimizer of $\mathcal{L}_1(f)$:
\begin{align}
    \lim_{\lambda \rightarrow 0}~~ f_\theta^\lambda \in  \arg \inf_{f \in \mathcal{F}} \mathcal{L}_1(f).
\end{align}
\end{definition}
Definition~\ref{def:continuity} formalizes the intuition that an addition of the loss $\mathcal{L}_2$ affects the optimal solution in a continuous manner, and does not abruptly change the solution. We would expect such a criterion to be satisfied when optimizing smooth loss functions (e.g., squared error) over parameters dictated by a neural network, with appropriate explicit regularization (e.g., $\ell_2$ regularization on parameters) or when the training procedure smoothens the solution using some form of implicit regularization.

We will assume that the function class obtained by composing $\mathcal{F}$ and $\Phi$, $\mathcal{F} \circ \Phi$, obeys the continuity assumption with respect to the  training objective in Equation~\ref{eqn:bi_level}. Concretely, in this case $\mathcal{L}_1$ and $\mathcal{L}_2$ are given by: $\mathcal{L}_1(f, \phi) = \mathbb{E}_{\bx, y}[(f(\phi(\bx)) - y)^2]$ and $\mathcal{L}_2(f, \phi) = \text{disc}_\mathcal{H}(\bbP_\pi(\phi(\bx)), \bbP_{\mu_\text{data}}(\phi(\bx)))$, and both of the loss functions are smooth with respect to the functions $\phi$ and $f$. In our experiments, we utilized a square loss for the discrepancy $\text{disc}_\mathcal{H}$, and therefore, it is smooth. Definition~\ref{def:continuity} in this case would reduce to the following condition:

\begin{assumption}[Definition~\ref{def:continuity} applied to $\mathcal{F} \circ \Phi$.]
\label{assumption:continuity}
For any given threshold $\delta \in \mathbb{R}$ such that $|\lambda| \leq \delta$, there exists a non-decreasing function $\zeta: \mathbb{R}^+ \rightarrow \mathbb{R}^+$ of $\lambda$ such that, $\zeta(0) = 0$ and $||f^\lambda_\theta \circ \phi_\theta^\lambda - f^* \circ \phi^*||_\infty \leq \zeta(\lambda)$ for some $(f^*, \phi^*) \in \arg \inf_{f \in \mathcal{F}, \phi \in \Phi} \mathcal{L}_1(f, \phi)$.   
\end{assumption}

Essentially, Assumption~\ref{assumption:continuity} suggests that if we can utilize an appropriately chosen value of $\lambda$, then we can roughly match the solution obtained by just minimizing $\mathcal{L}_1$, which is the prediction error in our case. This would allow us to control the discrepancy of a solution obtained by only optimizing the prediction error.

\subsection{Proof of Proposition~\ref{prop:representation_error}}

\begin{proposition}[Lower bounding the value under $\pi$.]
\label{prop:representation_error_restated}
Under the assumptions and criteria listed above, the ground truth objective for any given $\pi \in \Pi$ can be lower bounded in terms of the learned objective model $f_\theta \in \mathcal{F}$ and representation $\phi \in \Phi$, with high probability $\geq 1 - \delta$ as:
\begin{align}
\label{eq:plugging_back_in_appendix}
     J(\pi) -& {J}_\theta(\pi) \geq \underbrace{{J}(\mu_\text{data}) - {J}_\theta(\mu_\text{data})}_{(\blacksquare)}-\underbrace{C_\mathcal{F} \cdot \text{disc}_{\cH}(\bbP_{\mu_\text{data}}(\phi(\bx)), \bbP_{\pi}(\phi(\bx)))}_{(*)} - 2 \varepsilon'_{\mathcal{F}, \Phi} - \varepsilon_{\text{stat}}, \nonumber
\end{align}
where $C_{\mathcal{F}}$ and $\varepsilon'_{\mathcal{F}, \phi}$ are universal constants that only depend on the function class $\mathcal{F}, \Phi$ and other universal constants associated with the loss functions, and $\varepsilon_\text{stat}$ is a statistical error term that decreases as the dataset size $|\mathcal{D}|$ increases.  
\end{proposition}
\begin{proof}
Our high-level strategy would be to first decompose the terms into different components, and bound these components separately. We begin with the following decomposition:
\begin{align}
   & J(\pi) - {J}_\theta(\pi) \\ 
   &= \bbE_{\bx \sim \pi(\bx)}\left[f(\bx))\right] - \bbE_{\bx \sim \pi(\bx)}\left[{f}_\theta(\phi(\bx)))\right]\\
    &= \underbrace{\bbE_{\bx \sim \mu_\text{data}(\bx)}\left[f(\bx)\right] - \bbE_{\bx \sim \mu_\text{data}(\bx)}\left[{f}_\theta(\phi(\bx))\right]}_{(a)} \\
    &~~~~~~+ \underbrace{\bbE_{\bx \sim \pi(\bx)}\left[f(\bx))\right] - \bbE_{\bx \sim \pi(\bx)}\left[{f}_\theta(\phi(\bx)))\right] - \left( \bbE_{\bx \sim \mu_\text{data}(\bx)}\left[f(\bx)\right] - \bbE_{\bx \sim \mu_\text{data}(\bx)}\left[{f}_\theta(\phi(\bx))\right] \right)}_{:= \Delta(\pi, \mu_\text{data})} \nonumber.
\end{align}
Note that term (a) exactly corresponds to the first term $\blacksquare = J(\mu_\text{data}) - {J}_\theta(\mu_\text{data})$ in the expression in the proposition. To obtain the rest of the terms, we will now lower bound the remainder $\Delta(\pi, \mu_\text{data})$. First note that we can rearrange this term into a more convenient form:
\begin{align}
\Delta(\pi, \mu_\text{data}) :&= - \underbrace{\left(\bbE_{\bx \sim \pi(\bx)}\left[{f}_\theta(\phi(\bx))\right] - \bbE_{\bx \sim \mu_\text{data}(\bx)}\left[[{f}_\theta(\phi(\bx))\right] \right)}_{(i)} + \underbrace{\left(\bbE_{\bx \sim \pi(\bx)}\left[f(\bx)\right] - \bbE_{\bx \sim \mu_\text{data}(\bx)}\left[f(\bx)\right] \right)}_{(ii)}.
\end{align}
Term (i) in the above equation can directly be bounded to give rise to one of the terms that contributes to the discrepancy term the bound as shown below:
\begin{align}
    (i) :&= - \int_{\bx} \left(\pi(\bx) - \mu_\text{data}(\bx) \right) \cdot {f}_\theta(\phi(\bx))~ d\bx  \label{eqn:initial_eqn}\\
        &= - \int_{\bz} \left(\pi(\bz) - \mu_\text{data}(\bz) \right) \cdot {f}_\theta(\bz)~ d \bz \label{eqn:reparameterization}\\
 \implies (i) &\leq ||\widehat{f}||_\mathrm{L} \cdot \text{disc}_\cH(\bbP_\pi(\phi(\bx)), \bbP_{\mu_\text{data}}(\phi(\bx))) \label{eqn:block_end},
\end{align}
where the second step (Equation~\ref{eqn:reparameterization}) follows by reparameterizing the first step (Equation~\ref{eqn:initial_eqn}) in terms of the representation $\bz := \phi(\bx)$ and applying probability density change of variables simultaneously with the change of variables for integration, which leads to a cancellation of the Jacobian terms. $\bbP_{\mu_\text{data}}(\phi(\bx))$ denotes the probability distribution of the representations $\phi(\bx)$ when $\bx \sim \mu_\text{data}(\bx)$, and $\bbP_\pi(\phi(\bx))$ denotes the probability distribution of the representation $\phi(\bx)$ when $\bx \sim \pi(\bx)$. The last step follows from the fact that the discrepancy measure, $\text{disc}_\cH(p, q)$, which is an integral probability metric (IPM)~\citep{sriperumbudur2009integral}, upper bounds the total-variation divergence between $p$ and $q$. 

Now we tackle term (ii). First note that under the $\varepsilon$-realizability assumption, we note that there exists a $g \in \mathcal{F}, \phi^* \in \Phi$ such that $f(\bx) \approx g(\phi^*(\bx))$ in expectation under $\mu_\text{data}$ and $\pi$ (Assumption~\ref{assumption:realizability} holds for any distribution $\Pi$). That is,
\begin{align}
    (ii) &\leq  \left|\bbE_{\bx \sim \pi(\bx)}\left[g(\phi^*(\bx))\right] - \bbE_{\bx \sim \mu_\text{data}(\bx)}\left[g(\phi^*(\bx))\right] \right| + 2 \varepsilon_{\mathcal{F}, \Phi} \\
    &\leq ||g||_\mathrm{L} \left|\left|\bbE_{\bx \sim \pi(\bx)}\left[\phi^*(\bx)\right] - \bbE_{\bx \sim \mu_\text{data}(\bx)}\left[\phi^*(\bx)\right] \right|\right| + 2 \varepsilon_{\mathcal{F}, \Phi} \label{eqn:lipschtiz}\\
    &\leq ||g||_\mathrm{L} \cdot \text{disc}_\cH(\bbP_\pi(\phi^*(\bx)), \bbP_{\mu_\text{data}}(\phi^*(\bx))) + 2 \varepsilon_{\mathcal{F}, \Phi},
\end{align}
where the second step follows from the Lipschitzness of $\mathcal{F}$ and the third step follows similar to Equation~\ref{eqn:block_end}.

Finally, we will show that for an appropriately chosen $\lambda > 0$, $\text{disc}_\cH(\bbP_\pi(\phi^*(\bx)), \bbP_{\mu_\text{data}}(\phi^*(\bx)))$ is close to $\text{disc}_\cH(\bbP_\pi(\phi(\bx)), \bbP_{\mu_\text{data}}(\phi(\bx)))$. Applying Assumption~\ref{assumption:continuity}, we get that, for any given $\lambda$, there exists a $\zeta$ such that:
\begin{align}
    ||f^\lambda_\theta \cdot \phi^\lambda - f^* \cdot \phi^*||_\infty &\leq \zeta(\lambda)\\
    \implies ||\mathcal{L}_2(f^\lambda_\theta, \phi^\lambda) - \mathcal{L}_2(f^*, \phi^*)||_\infty &\leq C_{\mathcal{L}_2} \cdot \zeta(\lambda),
\end{align}
where the second step follows from the fact that the loss functions, $\mathcal{L}_1$ and $\mathcal{L}_2$  (in this case, $\text{disc}_\mathcal{H}$ and $(\cdot)^2$) are smooth functions of their arguments, and $C_{\mathcal{L}_2} := C_{\mathcal{H}}$ is the coefficient of smoothness when $\mathcal{L}_2 = \text{disc}_\mathcal{H}$. Thus, we get that:
\begin{align}
\label{eqn:disc_relation}
    \left| \text{disc}_\cH(\bbP_\pi(\phi^\lambda(\bx)), \bbP_{\mu_\text{data}}(\phi^\lambda(\bx))) - \text{disc}_\cH(\bbP_\pi(\phi^*(\bx)), \bbP_{\mu_\text{data}}(\phi^*(\bx))) \right| \leq C_\mathcal{H} \cdot \zeta(\lambda).
\end{align}

Finally, we need to bound $\text{disc}_\cH(\bbP_\pi(\phi^\lambda(\bx)), \bbP_{\mu_\text{data}}(\phi^\lambda(\bx)))$ in terms of the discrepancy value for the $\phi$ learned during training. Note that $\phi$ is obtained by minimizing the objective in Equation~\ref{eqn:bi_level}, which is the sample-based version of the objective used for obtaining $f^\lambda$. Using the uniform concentration argument from Lemma~\ref{lemma:uniform_conc}, we can upper bound $\mathcal{L}(f^\lambda, \phi^\lambda)$ in terms of $\mathcal{L}(f_\theta, \phi)$:
\begin{align*}
    \bbE_{\bx, y}[f^\lambda(\phi^\lambda(\bx)) - y)]^2  &+ \lambda \text{disc}_\cH(\bbP_\pi(\phi^\lambda(\bx)), \bbP_{\mu_\text{data}}(\phi^\lambda(\bx))) \lesssim \\ 
    &\bbE_{\bx, y}[f_\theta(\phi(\bx)) - y)]]^2  + \lambda \text{disc}_\cH(\bbP_\pi(\phi(\bx)), \bbP_{\mu_\text{data}}(\phi(\bx))) + \varepsilon_{\text{stat}}.
\end{align*}
Now, we can use the above relationship, alongside bounding $(f_\theta(\bx) - y)^2 \leq B_0$ (which holds since the function is bounded):
\begin{align}
    \label{eqn:uniform_conc_applied}
    \text{disc}_\cH(\bbP_\pi(\phi^\lambda(\bx)), \bbP_{\mu_\text{data}}(\phi^\lambda(\bx))) \leq \text{disc}_\cH(\bbP_\pi(\phi(\bx)), \bbP_{\mu_\text{data}}(\phi(\bx))) + C_\mathcal{H} \cdot \zeta + \frac{2B_0}{\lambda},
\end{align}
where $\zeta \rightarrow 0$ as $\lambda \rightarrow 0$. Now, we can minimize the right hand side w.r.t. $\lambda$ and obtain the desired result: 
\begin{align*}
    J(\pi) - J_\theta(\pi) \geq J(\mu_\text{data}) - J_\theta(\mu_\text{data}) - (2 C + 1) \cdot  \text{disc}_{\cH}(\bbP_{\mu_\text{data}}(\phi(\bx)), \bbP_{\pi}(\phi(\bx))) - 2 \varepsilon'_{\mathcal{F}, \Phi} - \varepsilon_{\text{stat}},
\end{align*}
where $\varepsilon'_{\mathcal{F}, \Phi} = \varepsilon_{\mathcal{F}, \Phi} + h(B_0, C_\mathcal{H})$, where $h(\cdot)$ is a bias that is a property of the loss functions, the function classes $\mathcal{F}$ and $\Phi$, and the discrepancy metric. It does not depend on the learned functions $f_\theta$, $\phi$ learned or the distribution $\pi$.
This completes the proof.

\end{proof}

\begin{remark}[\textbf{The interplay between $\lambda_0$, $\mathcal{F}$ and $\Phi$.}] In the proof above, we appeal to the continuity assumption (Assumption~\ref{assumption:continuity}) to enable us to bound the discrepancy measured in terms of the representation $\phi^*$ with respect to $\phi$, and this requires adjusting $\lambda$. This is unavoidable in the worst case if the function classes $\mathcal{F}$ and $\Phi$ are arbitrary and do not contain any solution $\phi^*$ that attains a small training error, while being somewhat invariant. In such cases, the only way to attain a lower-bound is to not be invariant, and this is what our bound would suggest by setting $\lambda$ to be on the smaller end of the spectrum. On the other hand, if the function class is more expressive, such that there exists a tuple $(f^*, \phi^*)$ that minimizes $\mathcal{L}_1(f, \phi))$ while being invariant, i.e., $\mathcal{L}_2(f^*, \phi^*)$ is small, then we can loosen this restriction over $\lambda$. 
\end{remark}

\subsection{Proof of Proposition~\ref{proposition:improvement}}

\begin{proposition}[Performance guarantee for \methodname] 
\label{proposition:improvement_restated}
Under Assumption~\ref{assumption:realizability}, the expected value of the ground truth objective under $\Opt$, $J(\Opt)$ is lower bounded by:
\begin{align*}
    J(\Opt) \gtrsim J(\mu_\text{data})  - \mathcal{O}\left(\sqrt{\frac{\log \frac{|\mathcal{F}| |{\Phi}| |{\Pi}|}{\delta}}{|\mathcal{D}|}} + \frac{\log \frac{|\mathcal{F}| |{\Phi}| |{\Pi}|}{\delta}}{|\mathcal{D}|} \right) + \underbrace{{J}_\theta(\Opt) - {J}_\theta(\mu_\text{data})}_{(\circ)} - (*) - \varepsilon'_{\mathcal{F}, \Phi}. 
\end{align*}
\end{proposition}

We restate the proposition above. There is a typo in the main paper, where the term $\varepsilon'_{\mathcal{F}, \Phi}$ got deleted from the bound when moving from Proposition~\ref{prop:representation_error} to \ref{proposition:improvement}. The above restatement of Proposition~\ref{proposition:improvement} provides the complete statement.

\begin{proof}
To prove this proposition, we can utilize the result from Proposition~\ref{prop:representation_error_restated}, but now we need to reason about a specific $\pi = \Opt$, where $\Opt$ depends on the learned $f_\theta$ and $\phi$. To account for such a $\pi$, we can instead bound the improvement $J(\pi) - J_\theta(\pi)$ for the worst possible $\pi \in \Pi$. While this does not change any term from the bound in Proposition~\ref{prop:representation_error_restated} that depends on the properties of the function class $\mathcal{F}$ or the representation class $\Phi$, this would still affect the statistical error, which would now require a uniform concentration argument over the policy class $\Pi$. Therefore, the bound from Proposition~\ref{prop:representation_error_restated} becomes:
\begin{align}
    J(\Opt) - J_\theta(\Opt) \geq J(\mu_\text{data}) - J_\theta(\mu_\text{data}) - C' \cdot  \text{disc}_{\cH}(\bbP_{\mu_\text{data}}(\phi(\bx)), \bbP_{\Opt}(\phi(\bx))) - 2 \varepsilon'_{\mathcal{F}, \Phi} - \varepsilon_{\text{stat}},
\end{align}
where $\varepsilon_\text{stat}$ is given by:
\begin{align}
     \varepsilon_\text{stat} = \mathcal{O} \left( \sqrt{\frac{\log \left(\frac{|\mathcal{F}| |\Phi| |\Pi| c_0}{\delta}\right)}{|\mathcal{D}|}} \right).
\end{align}
The term $\varepsilon_\text{stat}^2$ in the statement is a higher order term and is always dominated by $\varepsilon_\text{stat}$. Rearranging the terms completes the proof.
\end{proof}

\subsection{{Extension to any re-weighting of the dataset}}
{In this section, we will extend the theoretical result to compare the performace of the optimized design to the average objective value of any re-weighting of the dataset, and this by definition covers the comparison against the best design in the dataset. To do so, note that Proposition~\ref{prop:representation_error_restated} can be restated when comparing to $\mu_\text{r}$, a distribution induced by any specific re-weighting of the training dataset. This reformulation would give the following bound on the error:}
\begin{align}
\label{eq:reweighting}
     J(\pi) -& {J}_\theta(\pi) \geq \underbrace{{J}(\mu_\text{r}) - {J}_\theta(\mu_\text{r})}_{(\blacksquare)}-\underbrace{C_\mathcal{F} \cdot \text{disc}_{\cH}(\bbP_{\mu_\text{r}}(\phi(\bx)), \bbP_{\pi}(\phi(\bx)))}_{(*)} - 2 \varepsilon'_{\mathcal{F}, \Phi} - \varepsilon_{\text{stat}}.
\end{align}
{Since, the choice of $\mu_\text{r}$ is arbitrary (i.e., it can be any arbitrary weighting of the dataset distribution $\mu_\text{data}$), we can now basically compute the tighest lower bound by maximizing the RHS with respect to the choice of $\mu_\text{r}$. This would amount to:}
\begin{align}
\label{eq:reweighting_max}
     J(\pi) -& {J}_\theta(\pi) \geq \max_{\mu_\text{r} \in T}~~ \left[ \underbrace{{J}(\mu_\text{r}) - {J}_\theta(\mu_\text{r})}_{(\blacksquare)}-\underbrace{C_\mathcal{F} \cdot \text{disc}_{\cH}(\bbP_{\mu_\text{r}}(\phi(\bx)), \bbP_{\pi}(\phi(\bx)))}_{(*)} \right] - 2 \varepsilon'_{\mathcal{F}, \Phi} - \varepsilon_{\text{stat}},
\end{align}
{where $T$ denotes the set of distributions that arising from all possible re-weightings of the training data. For the special case when $\mu_\text{r}$ simply corresponds to a Dirac-delta distribution centered at the best point in the training dataset, which we will denote as $\mathbf{x}_\text{best}$, and $\mu_\text{r} = \delta_{\mathbf{x} = \mathbf{x}_\text{best}}$, we can write down the bound as:}
\begin{align}
\label{eq:final_bound}
     J(\Opt) \gtrsim f(\mathbf{x}_\text{best}) + \underbrace{{J}_\theta(\pi) - \widehat{f}_\theta(\mathbf{x}_\text{best})}_{\text{improvement}} - \underbrace{C_\mathcal{F} \cdot \text{disc}_{\cH}(\bbP_{\mu_\text{r}}(\phi(\bx)), \bbP_{\pi}(\phi(\bx)))}_{(*)} - 2 \varepsilon'_{\mathcal{F}, \Phi} - \varepsilon_{\text{stat}}.
\end{align}
{When converting Equation~\ref{eq:final_bound} to a performance guarantee, note that $\varepsilon_\text{stat}$, which is used to bound the discrepancy $\text{disc}_\mathcal{H}$ in Equation~\ref{eq:final_bound} would not decay as the size of the dataset $|\mathcal{D}|$ increases, as our reference point is just one single point $\mathbf{x}_\text{best}$ in the dataset. That said, note that as the dataset size increases, the groundtruth value of $\mathbf{x}_\text{best}$ shall also increase which would tighten the RHS of Equation~\ref{eq:final_bound} as the dataset size increases.}

{Therefore, our bound presents a tradeoff: when comparing $J(\Opt)$ to the average value of the dataset $J(\mu_\text{data})$, we can attain favorable guarantees that guarantee improvement as the dataset size increases because the statistical error reduces, and on the other hand, when we compare $J(\Opt)$ to $f(\mathbf{x}_\text{best})$, the statistical error may not decay as the dataset size increases, however, the function value $f(\mathbf{x}_\text{best})$ will increase. In summary, our bound in either case is expected to improve as the dataset size increases.}

{\textbf{Remark:} Proposition~\ref{proposition:improvement} is inspired from theoretical results in offline reinforcement learning, where it is common to lower bound the improvement of the learned policy over the behavior policy (see for example, works on ``safe policy improvment''~\citep{laroche2017safe,simao2019safe,kumar2020conservative,yu2021combo,cheng2022adversarially}). When comparing to the best point in the dataset, or equivalently, the best in-support policy in the case of reinforcement learning, the bounds would inevitably become weaker, since it is not clear if \emph{any} offline MBO or offline RL algorithm would improve over the best point in the dataset, in the worst-case problem instance. We have added this as a limitation of the theoretical results, but would remark that this limitation is also existent in other areas where one must learn to optimize decisions from offline data. Nevertheless, our empirical results do show that \methodname\ does improve over the best point in the dataset, as shown in Table~\ref{tab:perf}.}

\section{Experiment Details}
\label{appsec: exp}
\subsection{Tasks and Datasets}
\paragraph{Tasks.} We provide brief description of the tasks we use for our empirical evaluation and the corresponding evaluation protocol below. More details can be found be found at~\citet{trabucco2021designbench}.
\begin{itemize}
    \item The Superconductor task aims at optimizing the superconducting materials which achieves high-critical temperatures. The original features are vectors containing specific physical properties, such as the mean atomic mass of the elements. Following prior work~\citep{trabucco2021designbench}, we utilize the invertible input specification by representing different superconductors via a serialization of the chemical formula. The objective $y$ is the critical temperature the resulting material could achieve. Here, we have $\bx\in\bbR^{86}$ and $y\in\bbR^{+}$. The original dataset in Design-Bench has 21263 samples in total. However, following prior work~\citep{trabucco2021conservative,trabucco2021designbench,yu2021roma}, we remove the top $20\%$ from the dataset to increase the difficulty and ensure the top-performing points are not visible during learning stage. 
    {\textbf{Evaluation}: To evaluate the performance of any given design, we use the oracle function provided in design-bench, which trains a random forest model (with the default hyper-parameters provided in~\citep{hamidieh2018superconductor}) on the whole original dataset.}

    \item The Ant and D'Kitty Morphology task aim at designing the morphology of a quadrupedal robot such that the robot could be crawl quickly in certain direction. The Ant/D'Kitty task corresponds to the Ant/D'Kitty robot that used as the agent. Specifically, the input is the morphology of the agent with the objective value is the average return of the agent that uses the optimal control tailored for the morphology. More details about the controller design could be found at~\citet{trabucco2021designbench}. For the dimension of the design/feature space, we have $\bx \in \bbR^{60}$ and $\bx\in\bbR^{56}$ for D'Kitty. The original dataset contains 25009 samples for both Ant and D'Kitty. Following prior work~\citep{trabucco2021conservative, yu2021roma}, we remove the $40\%$ of top-performing samples and use the rest as the new dataset. {\textbf{Evaluation}: To evaluate any given morphology, it is passed into the MuJoCo Ant or D'Kitty environment, and a pre-trained policy is used (trained using Soft Actor-Critic for 10 million steps), the sum of rewards in a trajectory is used as the final score.}
    
    \item The Hopper-Controller task aims at designing a set of weights of a particular neural network policy that achieves high return for the MuJoCo Hopper-v2 environment. Basically, the policy architecture is a three-layer neural network with 64 hidden units and 5126 weights in total with the algorithm is trained using Proximal Policy Optimization~\citep{Schulman2017ppo} (with default parameters provided in~\citet{stable-baselines}). Specifically, the input is a vector of flattened weight tensors, with $\bx \in \bbR^{5126}$, the objective is the expected return from a single roll-out using this specific policy, with $y \in \bbR$. In the original dataset from Design-Bench, we have 3200 data points, and we do not discard any of them due to limited size of the dataset. {\textbf{Evaluation:} To evaluate the design of a specific set of weight, which deploy PPO policy with this set of weight, and the sum of rewards of 1000 steps are collected using the agent, as the final score.} 
\end{itemize}
\paragraph{Dataset.} For each task, we take the training dataset from design-bench, and in addition, randomly split it into training and validation with the ratio 7:3. The information of the original dataset is given in Table~\ref{table:dataset}.

\begin{table*}[hbt!]
    \centering
    \caption{Overview of the dataset information}
    \setlength\tabcolsep{4pt}
    \begin{tabular}{l||r|r|r|>{\columncolor[gray]{0.9}}r}
    \midrule
    \textbf{Dataset}            &  \textbf{Sample Size} & \textbf{Dimensions}& \textbf{Type of design space}\\
    \hline
    \textbf{Superconductor} & 17010 & 86 & \text{Continuous} \\
    \hline
    \textbf{Ant Morphology} & 10004 & 60 & \text{Continuous} \\
    \hline
    \textbf{D'Kitty Morphology} & 10004 & 56 & \text{Continuous} \\
    \hline
    \textbf{Hopper Controller} & 3200 & 5126 & \text{Continuous} \\
    \hline
    \textbf{ChEMBL} & 1093 & 31 & \text{Discrete} \\
    \hline
    \bottomrule
    \end{tabular}
    \label{table:dataset}
    \vspace{-.3cm}
\end{table*}

\subsection{Baseline Descriptions}
In this section, we briefly introduce the baseline methods for model-based optimization in an algorithmic level, all the implementation details for the baselines are referred to~\citet{trabucco2021designbench}.
\begin{itemize}
    \item \textbf{CbAS}~\citep{brookes19a} trains a density model via variational auto-encoder (VAE). Basically at each step, it updates the density model via a weighted ELBO objective which move towards better designs, which estimated by a pre-trained proxy model $\hat{f}(\bx)$; then generating new samples from the updated distributions to serve the new dataset, and repeated this process. 
    \item \textbf{Anto.CbAS}~\citep{fannjiang2020autofocused} is an improved version of the CbAS. For CbAS, it uses the fixed pre-trained proxy model $\hat{f}(\bx)$ all the time, to guide the distribution towards the optimal one. However, the proxy model will suffer from distributional shift as the density model changes. Hence, Auto.CbAS corrects the distributional shift by re-training the proxy model under the new optimized density model, via importance sample weights.
    \item \textbf{MINs}~\citep{kumar2019model} learns an inverse mapping $f^{-1}$ from the objective value $y$ to the corresponding design $\bx$. The inverse map is trained using the conditional GAN. Then it tries to search for the optimal $y$ value during the optimization, and the optimized design is found by sampling from this inverse map.
    \item \textbf{COMs}~\citep{trabucco2021conservative} learns a conservative proxy model $\hat{f}(\bx)$. Compared with the vanilla gradient ascent, COMs augments the regression objective with an additional regularization term which aims to penalize the over-estimation error of the out-of-distribution points. The final optimal designed point is found by performing gradient ascent on this learned proxy model.
    \item \textbf{RoMA}~\citep{yu2021roma} tries to overcome the brittleness of the deep neural networks that used in the proxy model. Basically, the method is a two-stage process, with the first stage is to train a proxy model $\hat{f}$ with Gaussian smoothing of the inputs, and the second stage is to update the solution while simultaneously adapt the model to achieve the smoothness in the input level for the current optimized point.  
    \item \textbf{BO-qEI}~\citep{reparameterization2017} refers the Bayesian Optimization with quasi-Expected Improvement acquisition function. The method tries to fit the proxy model $\hat{f}$, and maximizing it via Gaussian process.
    \item \textbf{CMA-ES}~\citep{Hansen06} refers the covariance matrix adaptation. The idea is to maintain a Gaussian distribution of the optimized design. At each step, samples of designs are drawing from the Gaussian distribution, and the covariance matrix is re-fitted by the top-performing designs estimated by a pre-trained proxy model $\hat{f}(\bx)$.
    \item \textbf{Grad./Grad.Min/Grad.Mean}~\citep{trabucco2021conservative} are a family of methods that use gradient ascent to find the optimal design. All of them require a pre-trained proxy model. For Grad., the proxy model is trained by minimizing the $L_2$ norm of $\hat{f}(\bx)$ and the ground-truth value $y$. For Grad.Min, it learns 5 different proxy models, and use the minimum of it as the predicted value. For Grad.Mean, it uses the mean of the 5 models. In the optimization stage, gradient ascent on the proxy model is used to find the optimized point.
    \item \textbf{REINFORCE}~\citep{Williams92} pre-trains a proxy model $\hat{f}(\bx)$ and use the policy gradient method to find the optimal design distributions that maximizes the $\hat{f}(\bx)$.
\end{itemize}

\subsection{Implementation Details}
For all the baseline methods, we directly report performance numbers associated with Design-Bench~\citep{trabucco2021designbench}. Below, we now discuss implementation details for the method developed in this paper, \methodname. Our implementation builds off of the official implementation of COMs, provided by the authors~\citet{trabucco2021conservative}

\textbf{Data preprocessing.} Following prior work~\citep{trabucco2021conservative}, we perform standardization for both the input $\bx$ and objective value $y$ before training the objective model $\widehat{f}$. Essentially, we standardize $\bx$ to $\tilde{\bx}:=\frac{\bx-\mu_{\bx}}{\sigma_{\bx}}$, with $\mu_{\bx}$ and $\sigma_{\bx}$ stands for the empirical mean and variance for $\bx$, similarly for $\tilde{y}$ with $\tilde{y}:=\frac{y-\mu_y}{\sigma_y}$.

\textbf{Neural networks representing $\widehat{f}$ and $\phi$.} We parameterize the learned model $\widehat{f}$ and the representation $\phi(\cdot) \in \mathbb{R}^{64}$ using three hidden layer neural networks, that assume identical architectures for all the tasks. The representation network, $\phi_\eta(\bx)$ is a neural network with two hidden layers of size 2048, followed by Leaky ReLU (using default leak value of 0.3). The representation is of size $128$. In some preliminary experiments, we observed that our method, \methodname\ was pretty robust to the dimension of the representation, so we simply chose to utilize $128$-dimensional representations for all the task, with no tuning. For the network representing $\widehat{f}$, we utilize a neural network with two hidden layers of size 1024, followed by Leaky ReLU (using default leak 0.3), the output is a 1-dimensional scalar used to predict the value $y$.

\textbf{Hyperparameters.} We utilize identical hyperparameters for all tasks, and these hyperparameters are taken directly from the implementation of \citet{trabucco2021conservative}. For training the learned model ${f}_\theta(\phi(\cdot))$, we use Adam optimizer with default learning rate $0.001$ and $(\beta_1, \beta_2) = (0.9, 0.999)$. The batch size is 128 and the number of training epochs is 300. For training the representation network, the optimizer is Adam with learning rate 0.0003 and default $(\beta_1, \beta_2) = (0.9, 0.999)$. We represent $\Opt$ as a non-parameteric distribution represented by 128 particles, initialized from 128 randomly sampled inputs $\bx$ in the dataset. We then alternate between performing one gradient step on $\Opt$ and one gradient step on $f_\theta$ and $\phi$. To select a value for the hyperparameter $\lambda$ (weight for the invariance regularizer term), we run our method with a wide range of $\lambda$ values  $\Lambda = \{0.1, 0.5, 1.0, 2.0, 5.0, 10.0, 100.0\}$. In the main experimental results we report, following the convention in prior work~\citep{trabucco2021designbench}, we choose $\lambda$ to be uniform across all tasks, but $\lambda$ is chosen via online evaluation. For our tuning procedure, we apply our tuning procedure where we keep the top 45\% models which achieve a
discriminator loss closes to 0.25, and select the one with the smallest validation prediction error among it . We perform an ablation of this strategy and the quantile of models that are kept in Section~\ref{sec:ablations}.

For the \conmethodname\ variant we show in our experiments, we set the initial value of the Lagrange multiplier $\alpha=0.3$--and this value is taken directly from \citet{trabucco2021conservative}--and update it using Adam optimizer with learning rate 0.01. 

We release the code for our method along with additional details at the following anonymous website: \url{https://sites.google.com/view/iom-neurips}.

\subsection{Additional Results}
In this section, we provide additional experimental results.
\label{subsec:app_result}

\subsection{Overall Performance Comparisons} 
In Table~\ref{tab:perf}, we provide the comparison of \methodname\ and \conmethodname\ with all baselines (including RoMA). To understand the trend in aggregated results, please check the plots in Figure~\ref{fig:perf_all}.
\begin{figure}[htp]
\vspace{-0.2cm}
\centering
    \includegraphics[height=5cm]{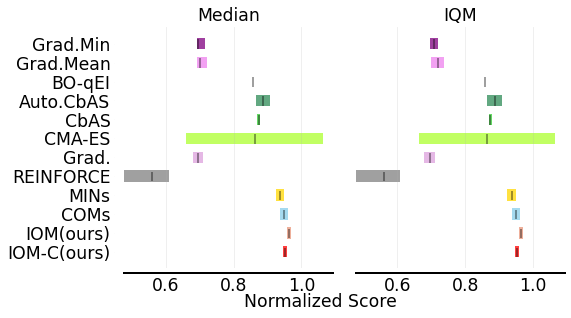}
    \caption{\footnotesize Median and IQM~\citep{agarwal2021deep} (with 95\% Stratified Bootstrap CIs) for the aggregated normalized score on Design-bench, for all the baseline methods. As we could not find the individual runs for RoMA~\citep{yu2021roma}, we report the mean and standard deviation for RoMA (copied from the results in~\citet{yu2021roma}) in Appendix~\ref{subsec:app_result}. }
    \label{fig:perf_all}
\end{figure}

\begin{table*}[hbt!]
    \centering
    \setlength\tabcolsep{4pt}
    \caption{\textbf{Comparative evaluation of \methodname\ and \conmethodname\ against prior methods} in terms of the mean 100th-percentile score (following the protocol in \citet{trabucco2021designbench}) and its standard deviation over 5 trials (the number for ROMA is copied from~\citet{yu2021roma}, and it is over 16 trails). {We have now additionally evaluated the LCB+GP baseline that optimizes the design against the lower-confidence bound estimate obtained from the posterior of a GP fit to the training data. Observe that \methodname\ outperforms this approach as well.}}
    \begin{tabular}{l||r|r|r|r}
    \midrule
    {}            &    \textbf{Superconductor} &     \textbf{Ant Morphology} &     \textbf{D'Kitty Morphology} &     \textbf{Hopper Controller}\\
    \midrule
    $\mathcal{D}$ (\textbf{best}) & 0.399 & 0.565 & 0.884 & 1.0 \\
    Auto. CbAS    & 0.421 $\pm$ 0.045 &    0.882 $\pm$ 0.045 &          0.906 $\pm$ 0.006 &      0.137 $\pm$ 0.005 \\
    CbAS          & \textbf{0.503 $\pm$ 0.069} &          0.876 $\pm$ 0.031 &          0.892 $\pm$ 0.008 &          0.141 $\pm$ 0.012\\
    MINs          & 0.469 $\pm$ 0.023 &          0.913 $\pm$ 0.036 &          0.945 $\pm$ 0.012 &          0.424 $\pm$ 0.166  \\
    COMs & 0.439 $\pm$ 0.033 & 0.944 $\pm$ 0.016 & 0.949 $\pm$ 0.015 & \textbf{2.056 $\pm$ 0.314}\\
    RoMA & \textbf{0.562 $\pm$ 0.030} & 0.875 $\pm$ 0.013 & \textbf{1.036 $\pm$ 0.042} & 1.867 $\pm$ 0.282 \\
    BO-qEI        & 0.402 $\pm$ 0.034 &          0.819 $\pm$ 0.000 &          0.896 $\pm$ 0.000 &          0.550 $\pm$ 0.118 \\
    CMA-ES        & 0.465 $\pm$ 0.024 &         \textbf{ 1.214 $\pm$ 0.732} &          0.724 $\pm$ 0.001 &          0.604 $\pm$ 0.215\\
    Grad.         & \textbf{0.518 $\pm$ 0.024} &          0.293 $\pm$ 0.023 &          0.874 $\pm$ 0.022 &          1.035 $\pm$ 0.482 \\
    Grad. Min     & 0.506 $\pm$ 0.009 &          0.479 $\pm$ 0.064 &          0.889 $\pm$ 0.011 &         \textbf{ 1.391 $\pm$ 0.589}\\
    Grad. Mean     & 0.499 $\pm$ 0.017 &          0.445 $\pm$ 0.080 &          0.892 $\pm$ 0.011 &         \textbf{ 1.586 $\pm$ 0.454}\\
    REINFORCE     & 0.481 $\pm$ 0.013 &          0.266 $\pm$ 0.032 &          0.562 $\pm$ 0.196 &         -0.020 $\pm$ 0.067 \\
    {LCB+GP}     & {N/A}  &   {0.765 $\pm$ 0.056}
       &    {0.871 $\pm$ 0.009}   &   {1.736 $\pm$ 0.261}     \\
    \midrule
    \textbf{\methodname\ (Ours)} & \textbf{0.504 $\pm$ 0.040} & 0.977  $\pm$  0.009 & 0.949 $\pm$  0.010 & \textbf{2.444 $\pm$ 0.080}\\ 
    \textbf{\conmethodname\ (Ours)} & \textbf{0.511 $\pm$ 0.037}  & 0.966  $\pm$ 0.012  & 0.940  $\pm$ 0.007  & \textbf{1.482  $\pm$ 0.650} \\
    \bottomrule
    \end{tabular}
    \label{tab:perf}
    \vspace{-.3cm}
\end{table*}

\section{Additional results for early stopping the optimization of $\Opt$}
We provide additional results for the effectiveness of our tuning strategy for early stopping. Results for the additional two tasks are presented in Figure~\ref{fig:app_checkpoint}. The top row shows the objective value of the optimized point as a function of the number of training epochs, and the bottom row shows the prediction error on the validation set. We use the vertical line to denote the point with the smallest MSE by our early stopping strategy. Among all tasks, our strategy is able to find a good checkpoint and this verifies the efficacy of our approach on different domains.

\begin{figure}[htp]
    \centering
    \includegraphics[width=0.99\textwidth]{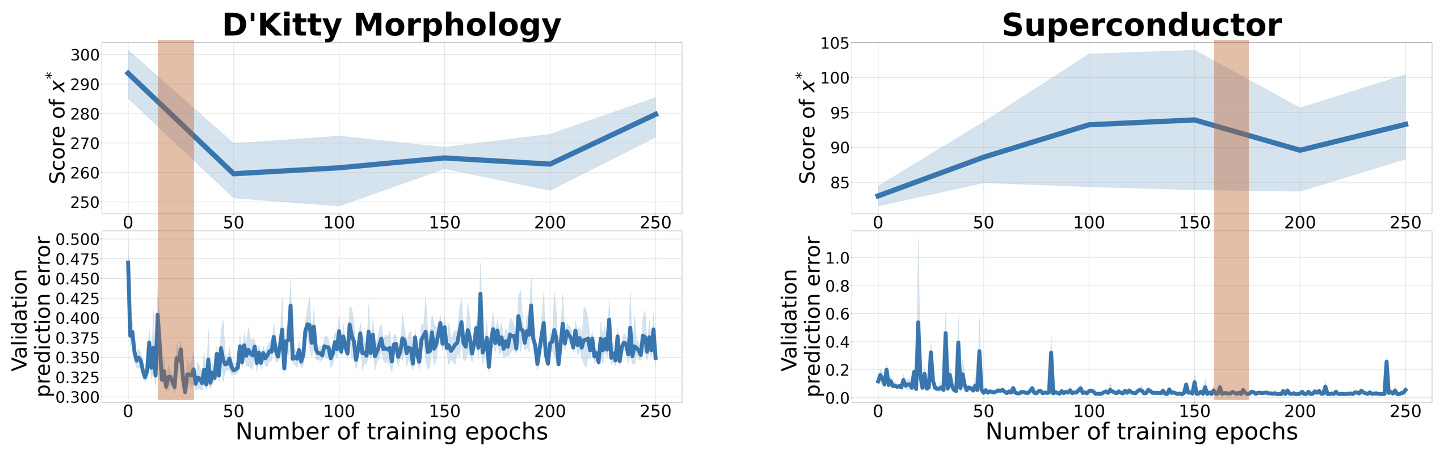}
    \caption{\footnotesize \textbf{Visualizing our early stopping scheme on two more tasks.} As per Section~\ref{sec:cross_validation}, the checkpoint selected is the epoch that has the smallest validation prediction error, and early stopping at this point generally performs well.}
    \label{fig:app_checkpoint}
\end{figure}
    
\subsection{Performance of \methodname\ and \conmethodname\ under different invariance regularizer $\lambda$}
\label{app:lambda}

In this section, we provide the final performance of \methodname\ and \conmethodname, under different values of the hyperparameter $\lambda$ used for training (Equation~\ref{eqn:training_ioms}). Results for all tasks are displayed in Figure~\ref{fig:all_scores}. As is the case with any other prior method for data-driven decision making that adds a regularizer, both \methodname\ and \conmethodname\ are also sensitive to the hyper-parameter $\lambda$, which demonstrates the necessity of our proposed tuning scheme. And fortunately, as observed in Section~\ref{sec:exp}, our tuning scheme can indeed attain good performance, for many choices of the $\epsilon$ (see Figure~\ref{fig:ablation}), making it easier to tune this hyperparameter $\lambda$ in practice. 

\begin{figure*}[ht]
    \vspace{-0.05in}
    \centering
    {\includegraphics[width=0.49\textwidth]{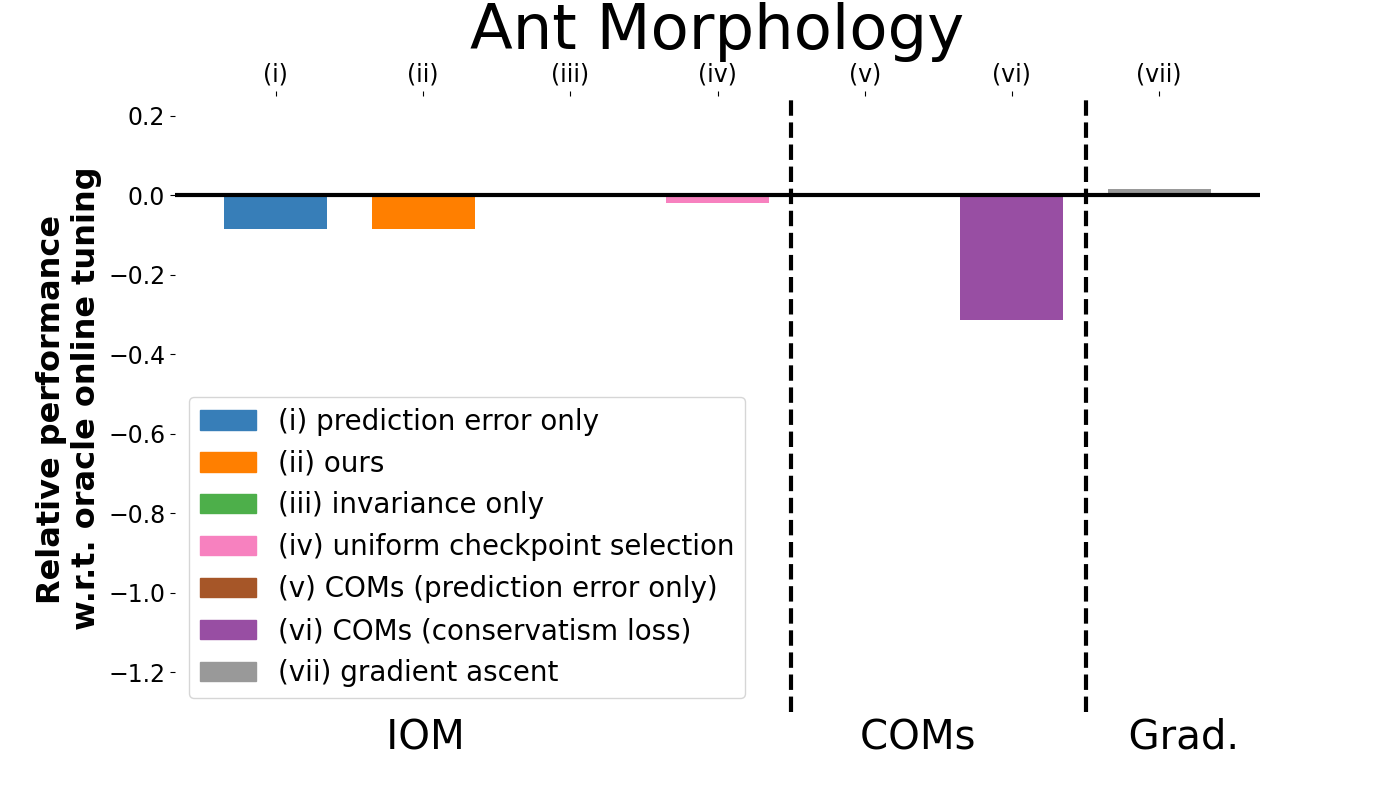}} 
    {\includegraphics[width=0.49\textwidth]{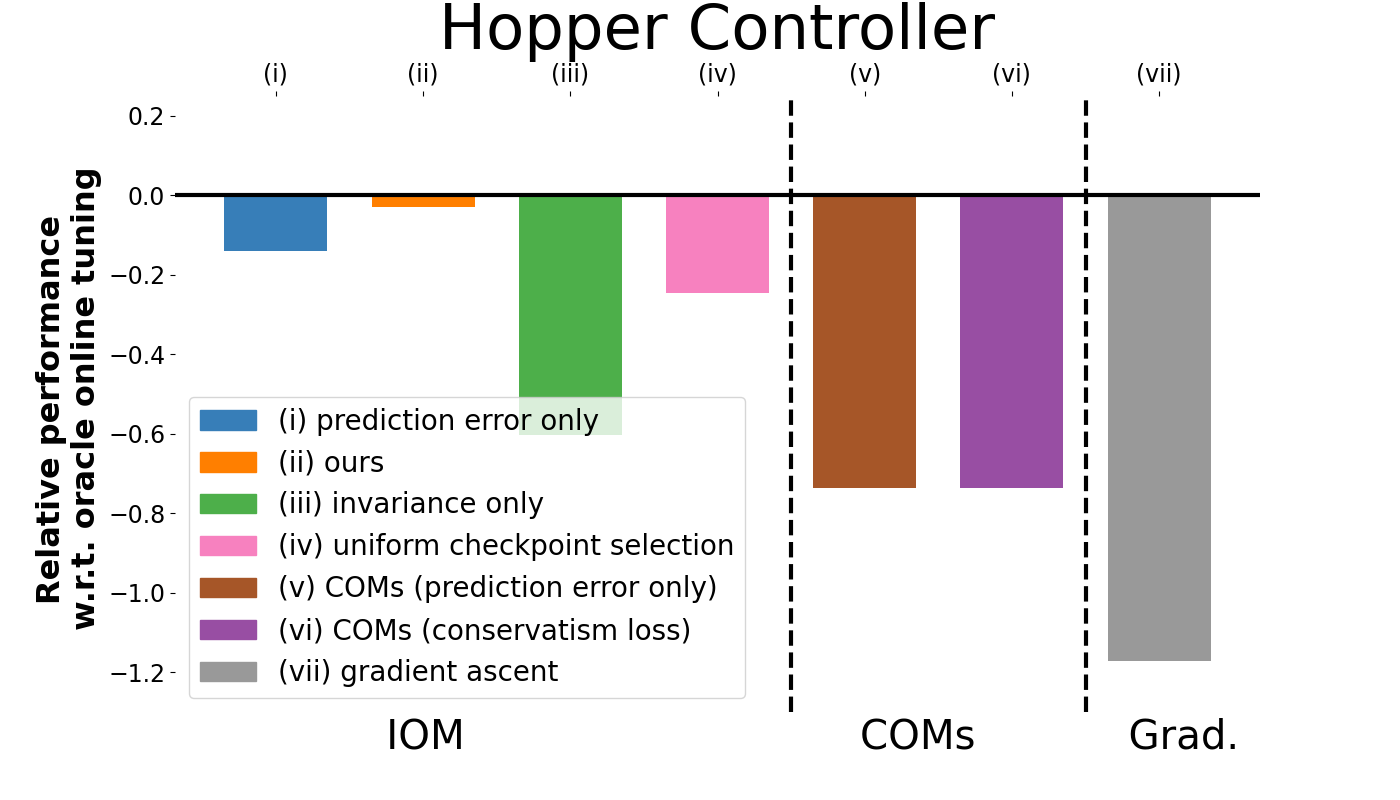}} 
    {\includegraphics[width=0.49\textwidth]{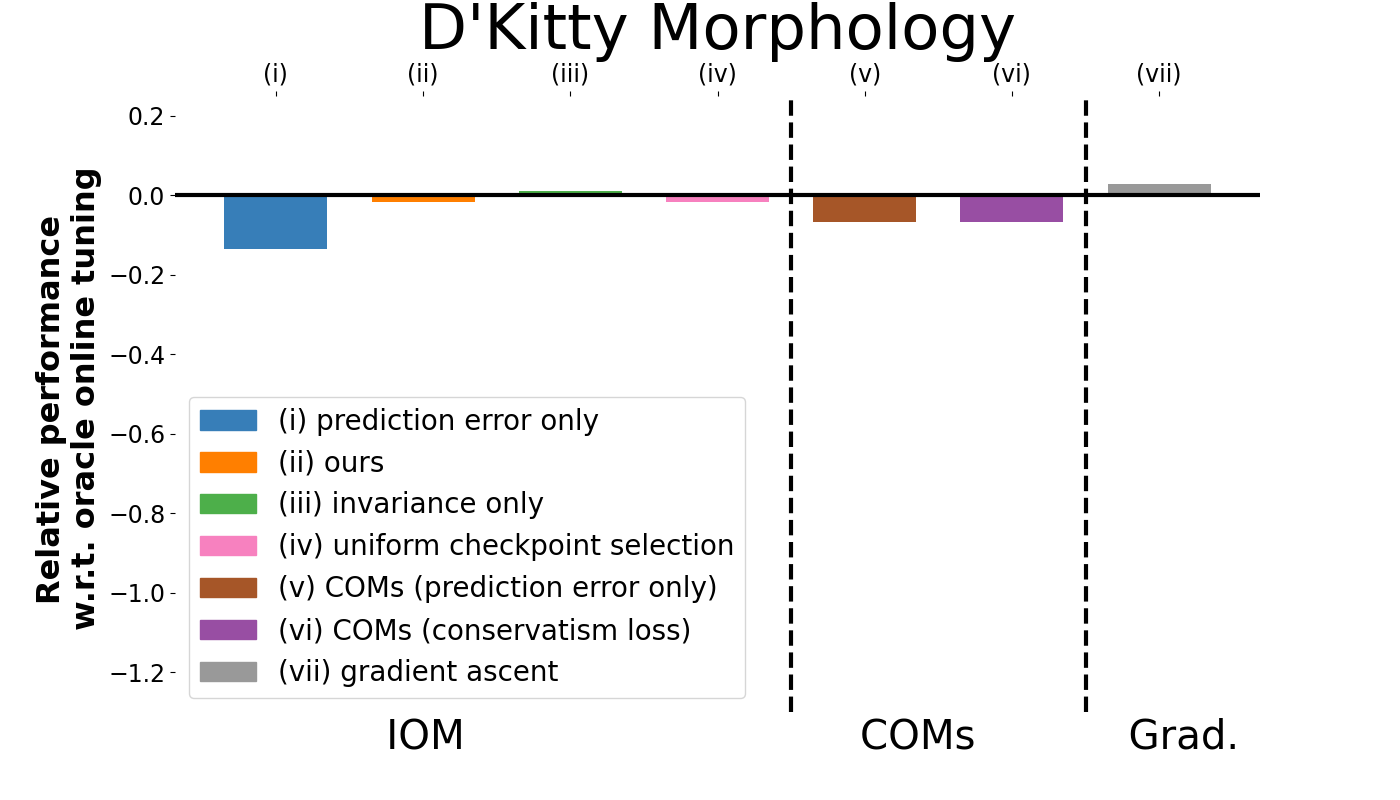}} 
    {\includegraphics[width=0.49\textwidth]{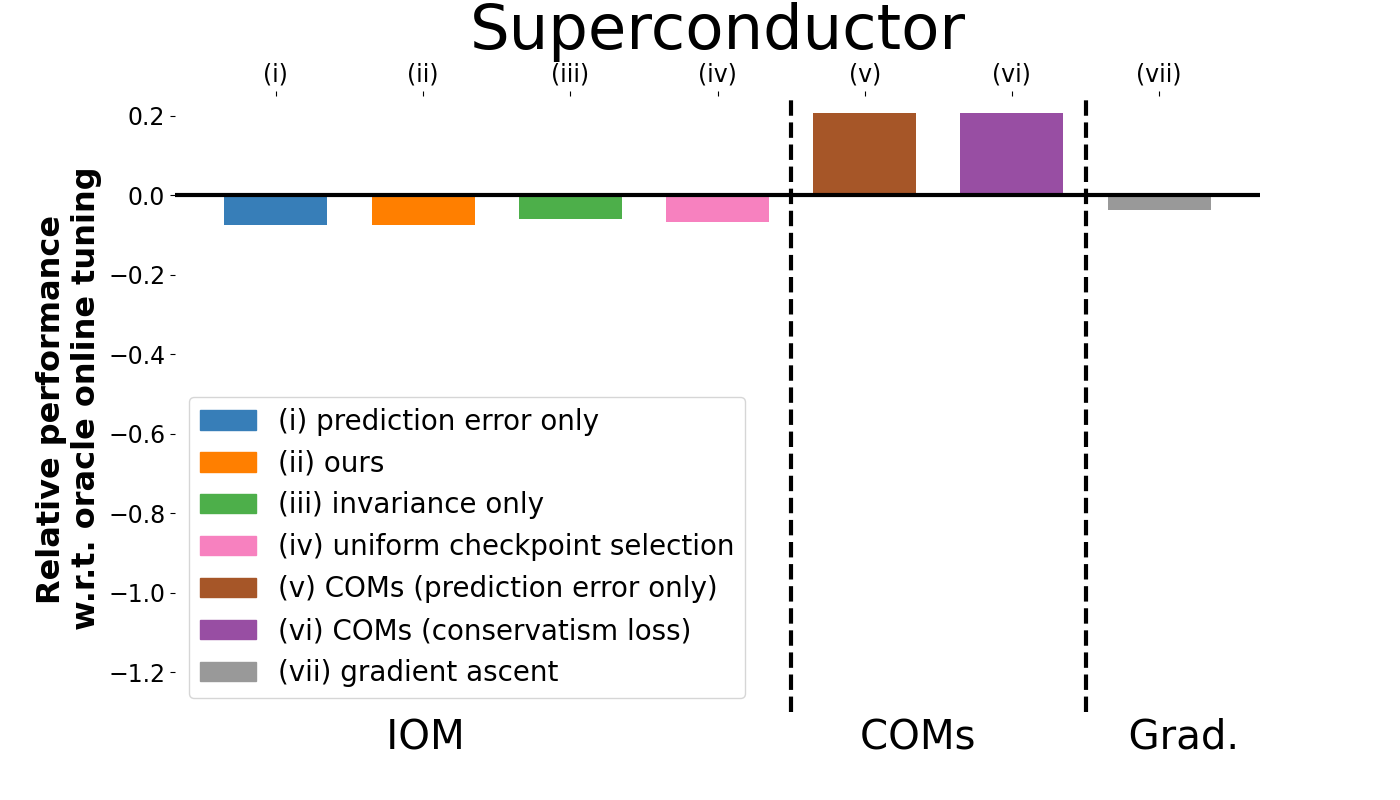}}
    
    \caption{Suboptimality gaps of various offline tuning strategies for \methodname\ when compared in each individual task.}
    \label{fig:mmd_all}
    \vspace{-0.15in}
\end{figure*}

\begin{figure*}
    \vspace{-0.05in}
    \centering
    {\includegraphics[width=0.48\textwidth]{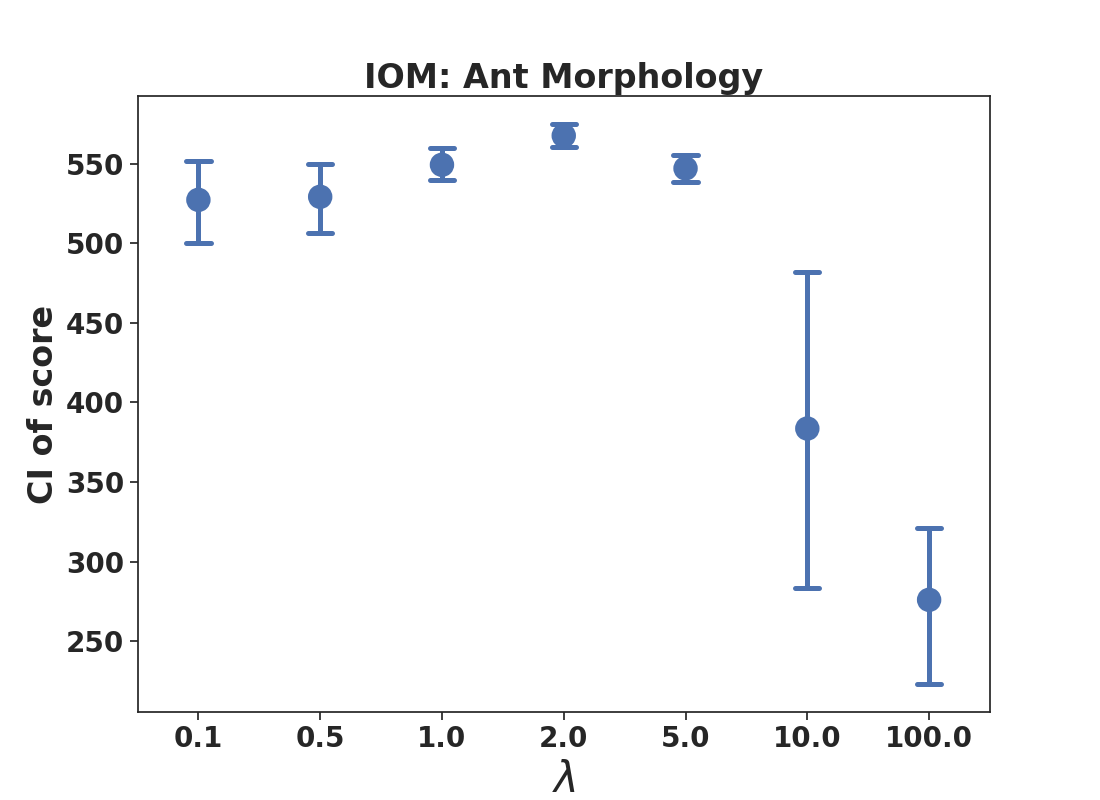}} 
    {\includegraphics[width=0.48\textwidth]{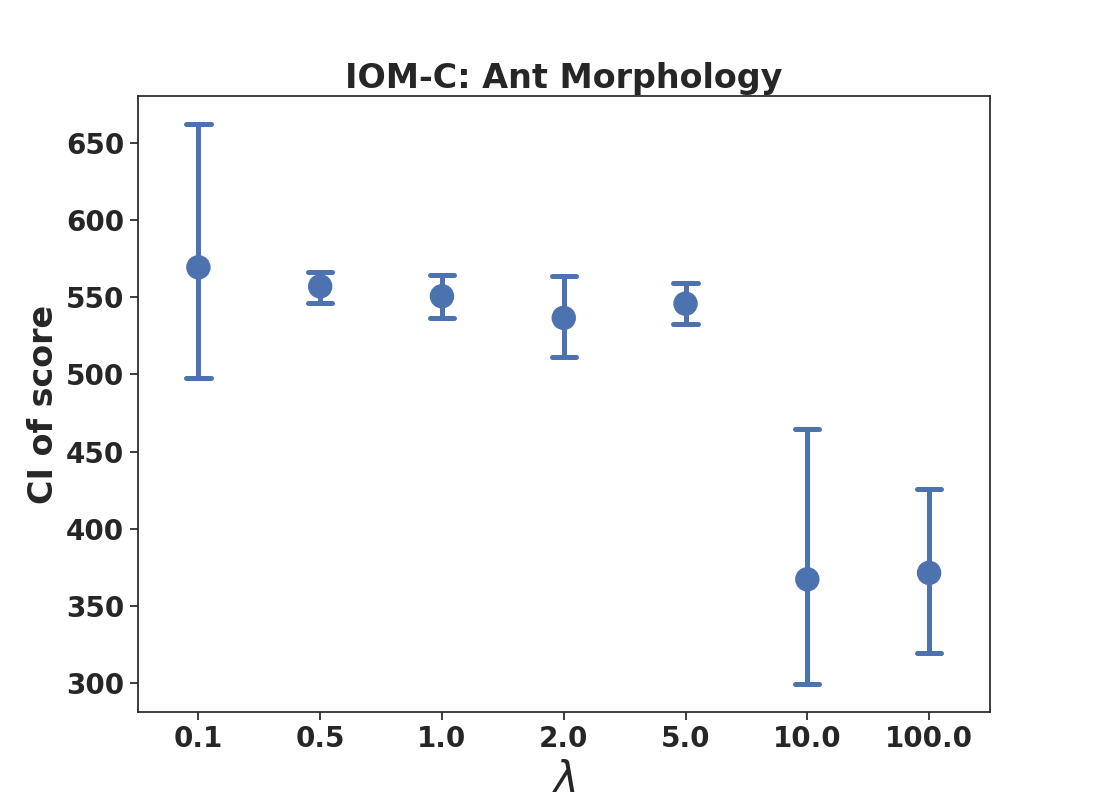}} 
    {\includegraphics[width=0.48\textwidth]{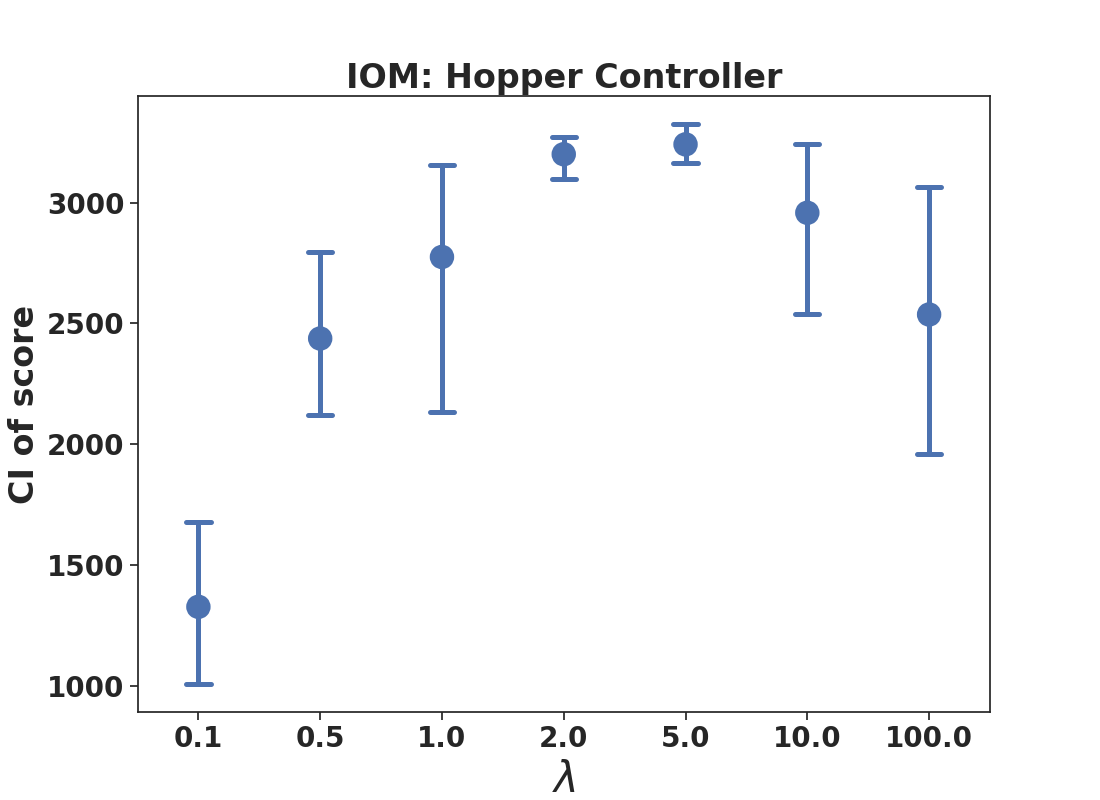}} 
    {\includegraphics[width=0.48\textwidth]{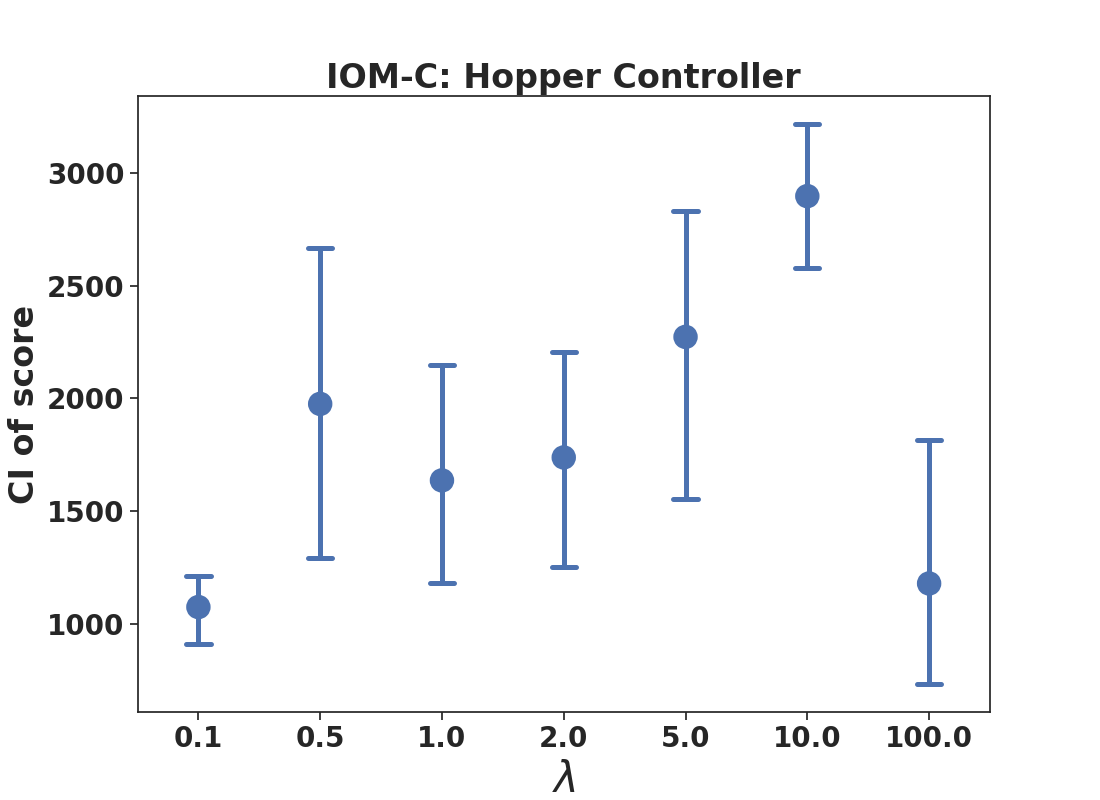}}
    {\includegraphics[width=0.48\textwidth]{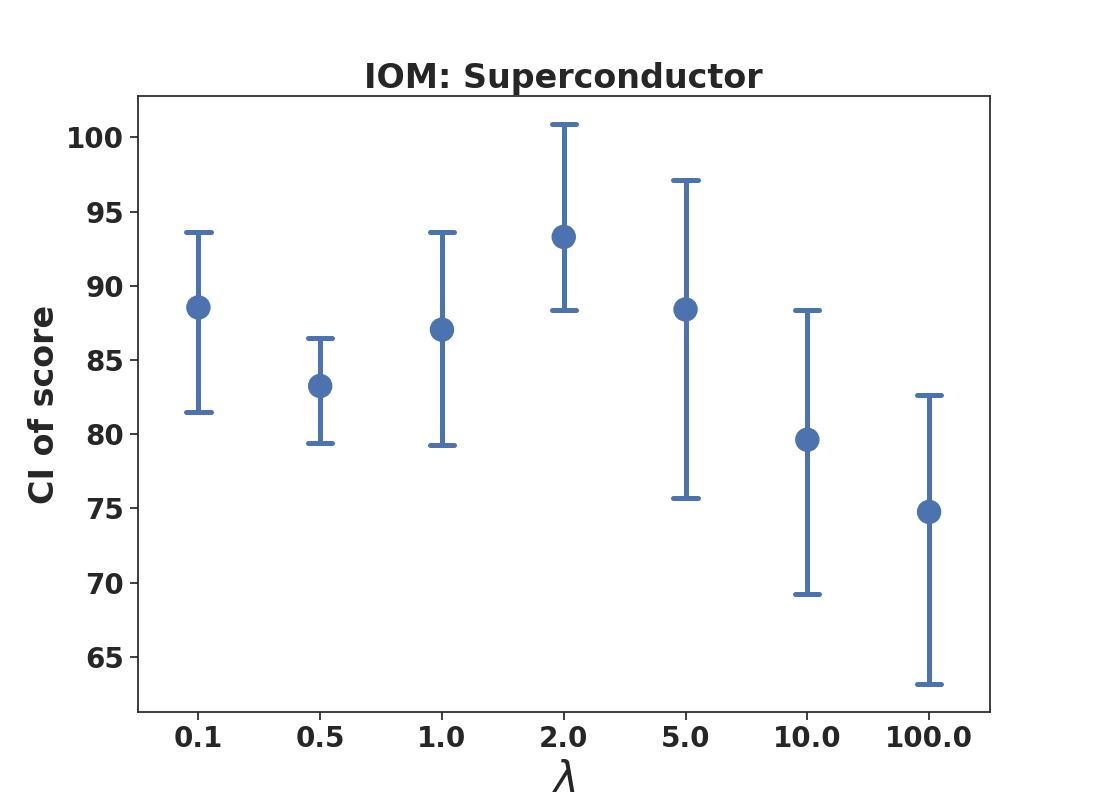}} 
    {\includegraphics[width=0.48\textwidth]{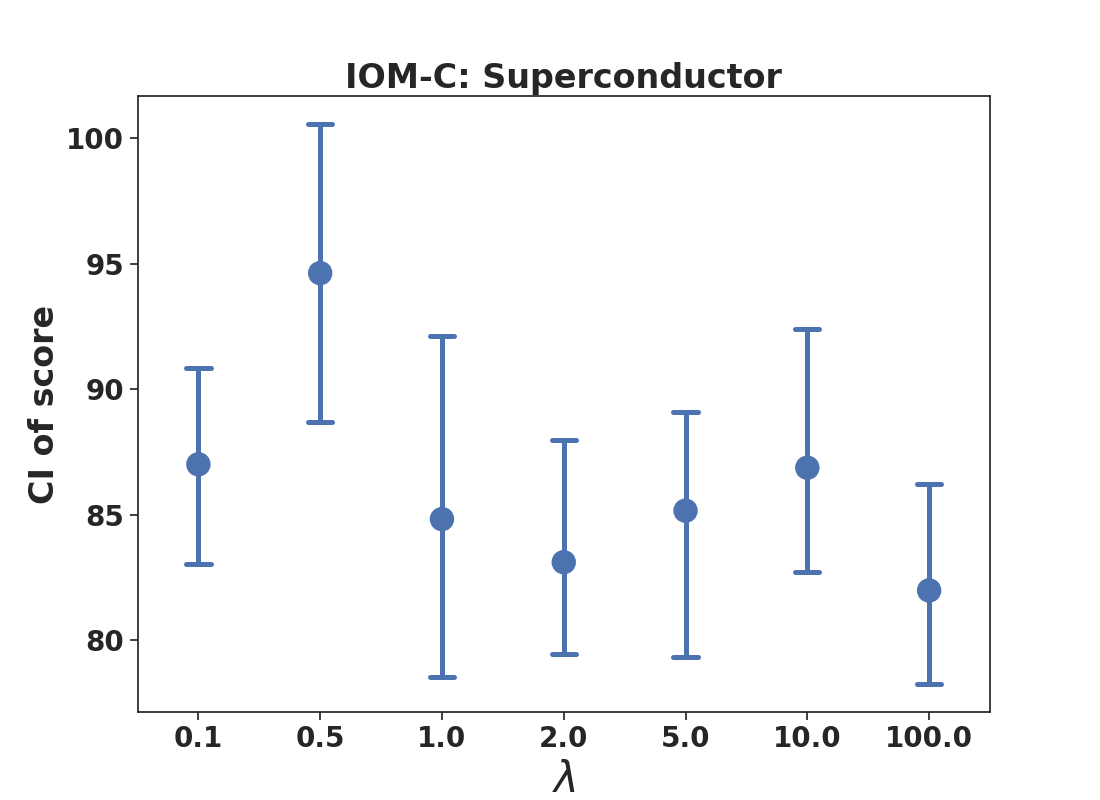}}
    {\includegraphics[width=0.48\textwidth]{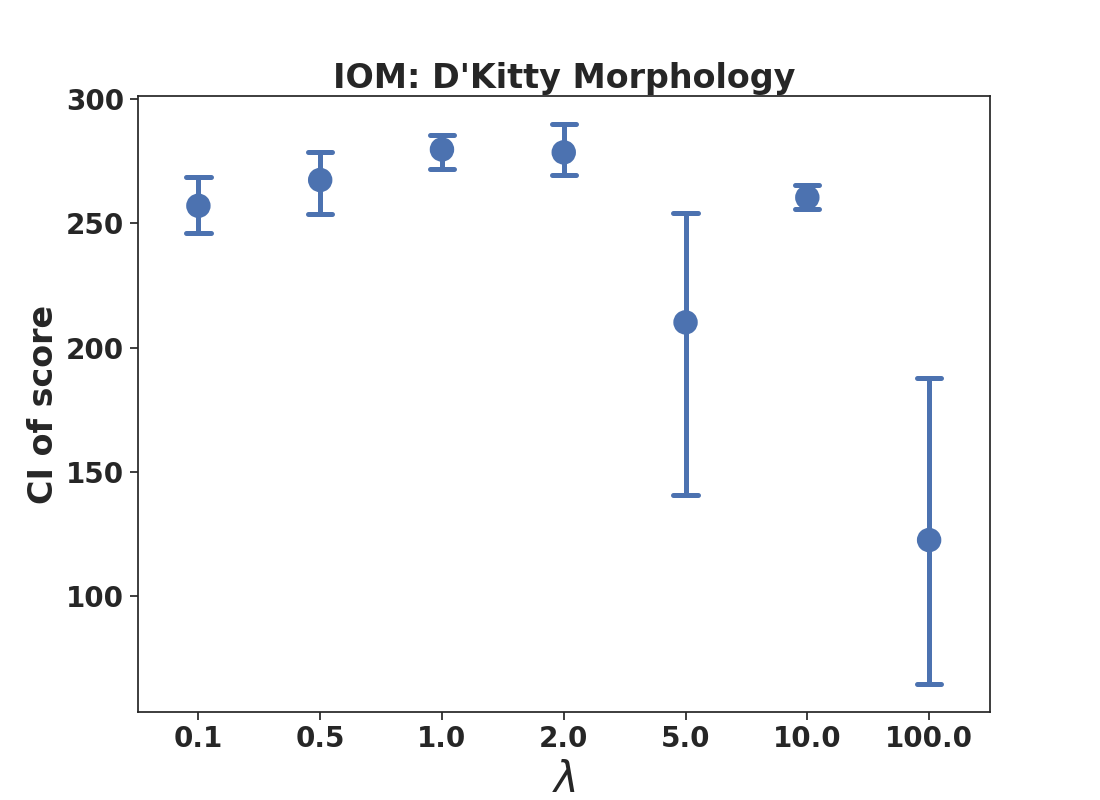}} 
    {\includegraphics[width=0.48\textwidth]{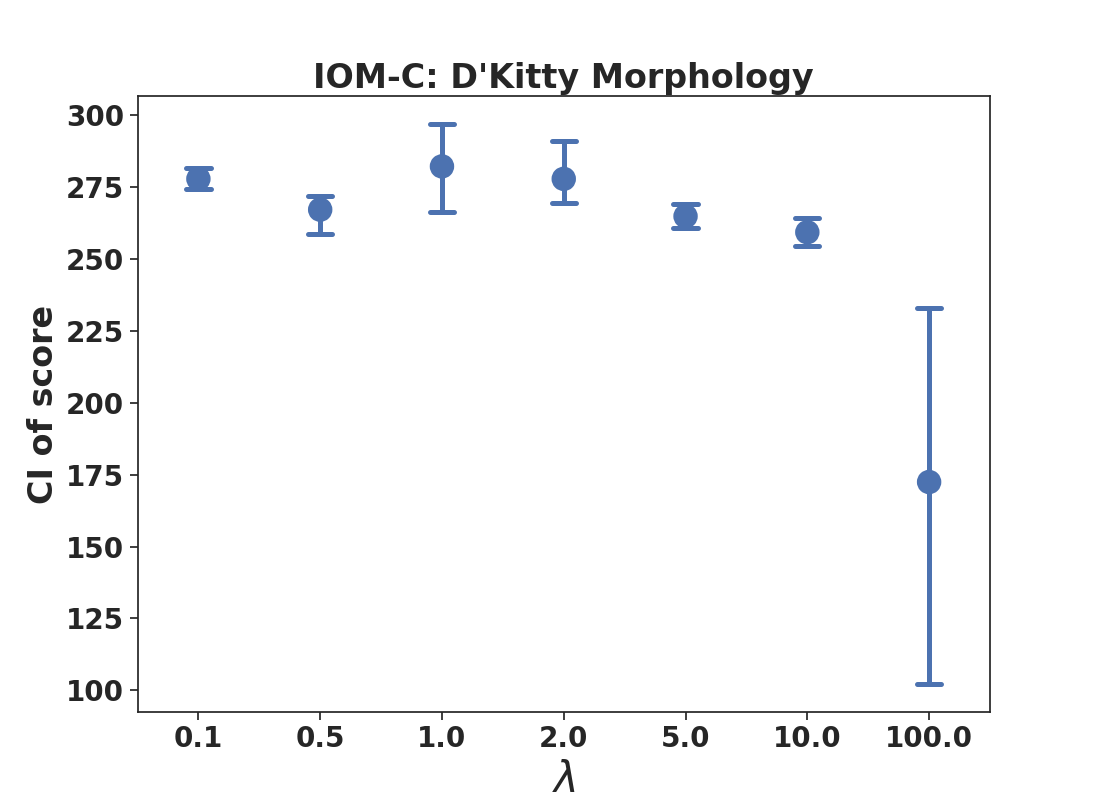}}
    \caption{Score of final optimized $\bx^*$ (with 2 standard deviations) of \methodname\ and \conmethodname\ under different invariance regularizer $\lambda$}
    \label{fig:all_scores}
    \vspace{-0.15in}
\end{figure*}

\newpage

\section{{Additional Results}}
\label{app: additional}

{\underline{\textbf{Computational cost of tuning \methodname\ via our offline workflow.}}} 
{In order to tune \methodname\ via our offline workflow we first run \methodname\ with multiple different hyperparameters. Concretely, in our experiments, we use 7 different hyperparameters, and then apply our offline tuning strategy, which only utilizes the loss values of the various runs, without incurring any significant computational cost of its own.}

{Thus, our hyperparameter tuning strategy does not incur additional computational cost. When it comes to the cost of running \methodname, we expect that beyond implementation differences, the cost of running \methodname\ should not be significantly worse than COMs: the only additional component trained by \methodname\ is the discriminator that is used to compute the $\chi^2$-divergence on the representation, which in our implementation only adds 2 more fully connected layers.}

{\underline{\textbf{COMs lead to worse training prediction error and more non-smooth functions compared to \methodname.}} Next we attempt to gain insights into what might explain why \methodname\ outperforms COMs in several of our tasks. We suspect that this is because applying a conservatism regularizer on the learned function value itself, like COMs~\citep{trabucco2021conservative}, does in fact distort the learned objective function more than \methodname. To understand if this is the case empirically, we plot two quantities: \textbf{(a)} the value of the training prediction error: $\mathbb{E}_{\mathbf{x}, y \sim \mathcal{D}}\left[ (\widehat{f}_\theta(\mathbf{x}) - y)^2 \right]$ for the best hyperparameter for COMs and \methodname\ in Figure~\ref{fig:distortion}, and \textbf{(b)} the ``smoothness'' of the learned function w.r.t. input, measured via the norm of the gradient of $f_\theta$ with respect to $\mathbf{x}$, i.e., $\mathbb{E}_{\mathbf{x} \sim \mathcal{D}}\left[ ||\nabla_x f_\theta(\phi(\bx)||_2^2 \right]$ in Figure~\ref{fig:grad_norm}. Observe that in both the Ant Morphology and Hopper Controller tasks, \methodname\ is able to minimize the prediction error better than COMs, and the gradient norm for COMs is strictly higher than \methodname\ indicating that conservatism on learned function leads to more non-smooth functions and reduces the ability of the learned function to fit the offline dataset.}

\begin{figure*}[ht]
    \vspace{-0.05in}
    \centering
    {\includegraphics[width=0.43\textwidth]{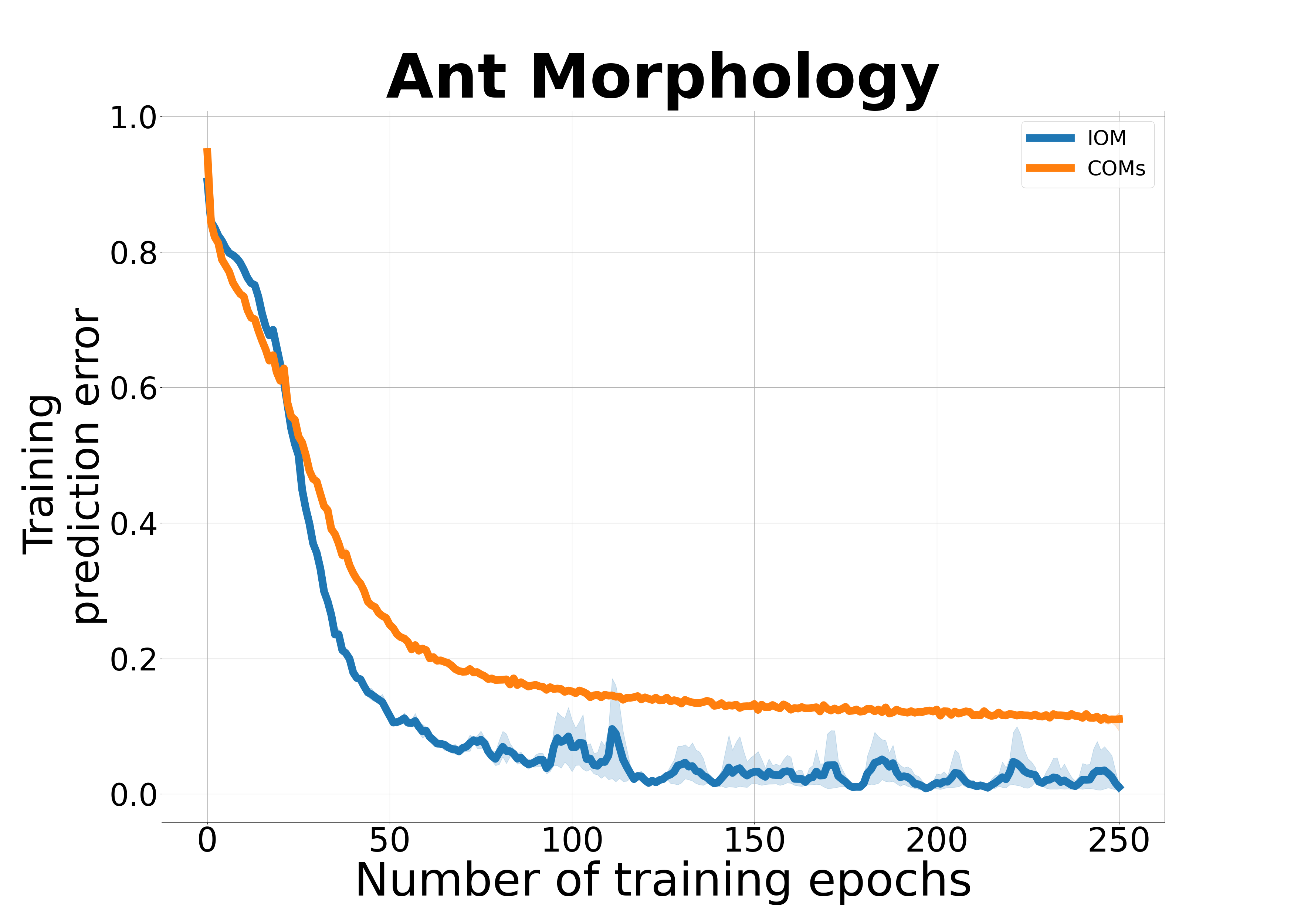}} 
    {\includegraphics[width=0.43\textwidth]{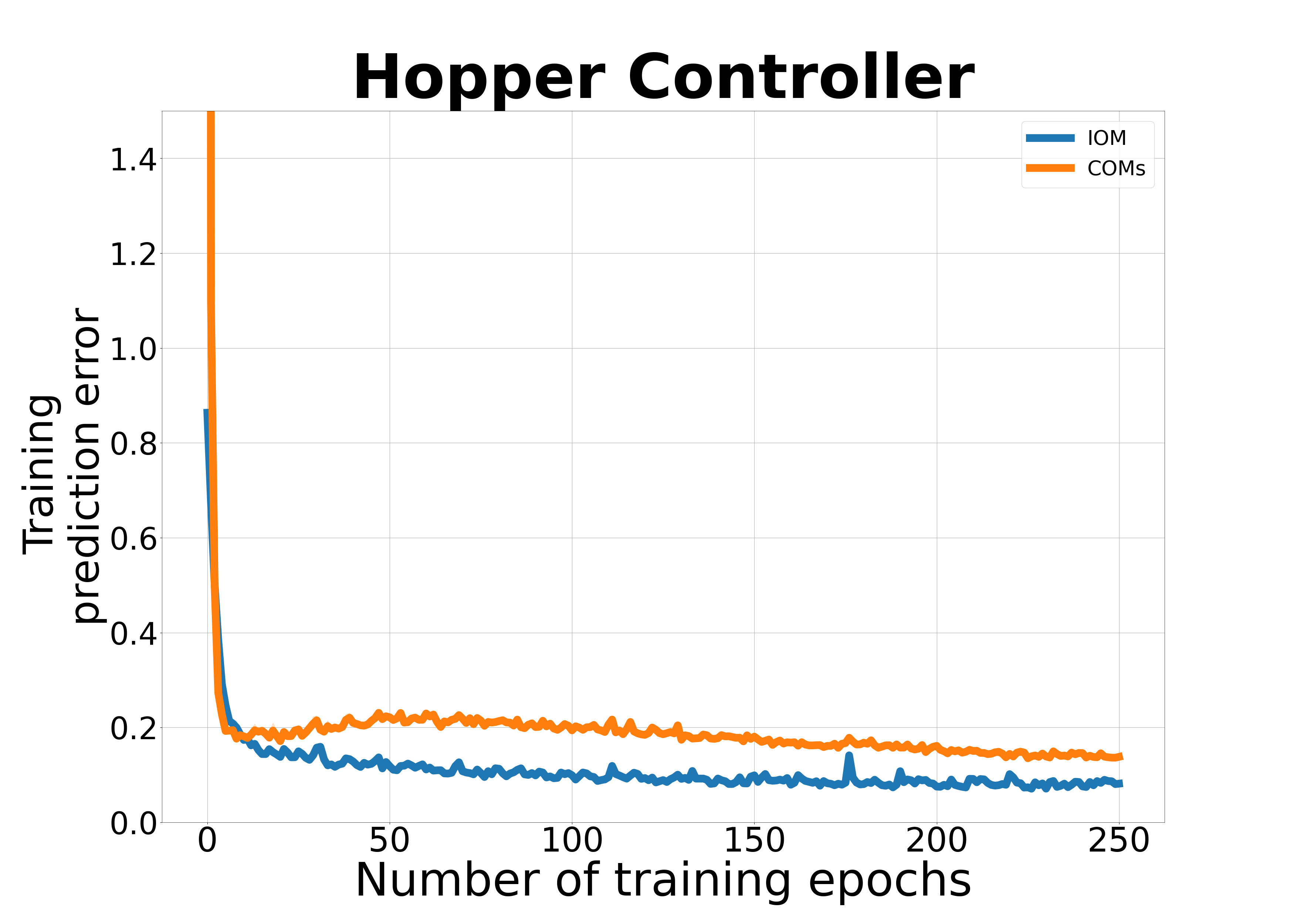}} 
    
    \caption{\footnotesize {\textbf{Comparison of prediction error between \methodname\ and COMs.} Note that COMs attains a worse prediction error than \methodname.}}
    \label{fig:distortion}
    \vspace{-0.15in}
\end{figure*}

\begin{figure*}[ht]
    \vspace{-0.05in}
    \centering
    {\includegraphics[width=0.43\textwidth]{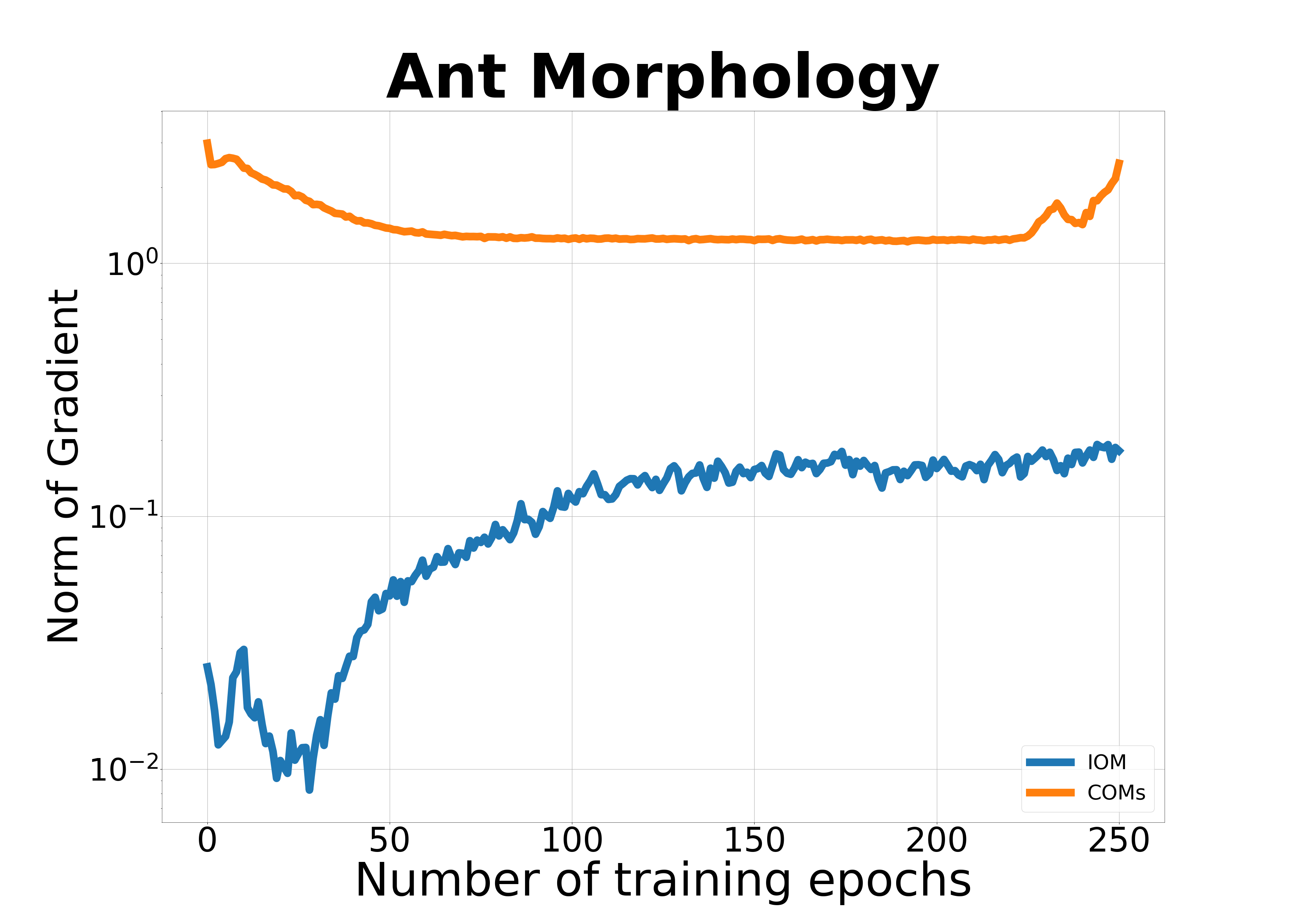}} 
    {\includegraphics[width=0.43\textwidth]{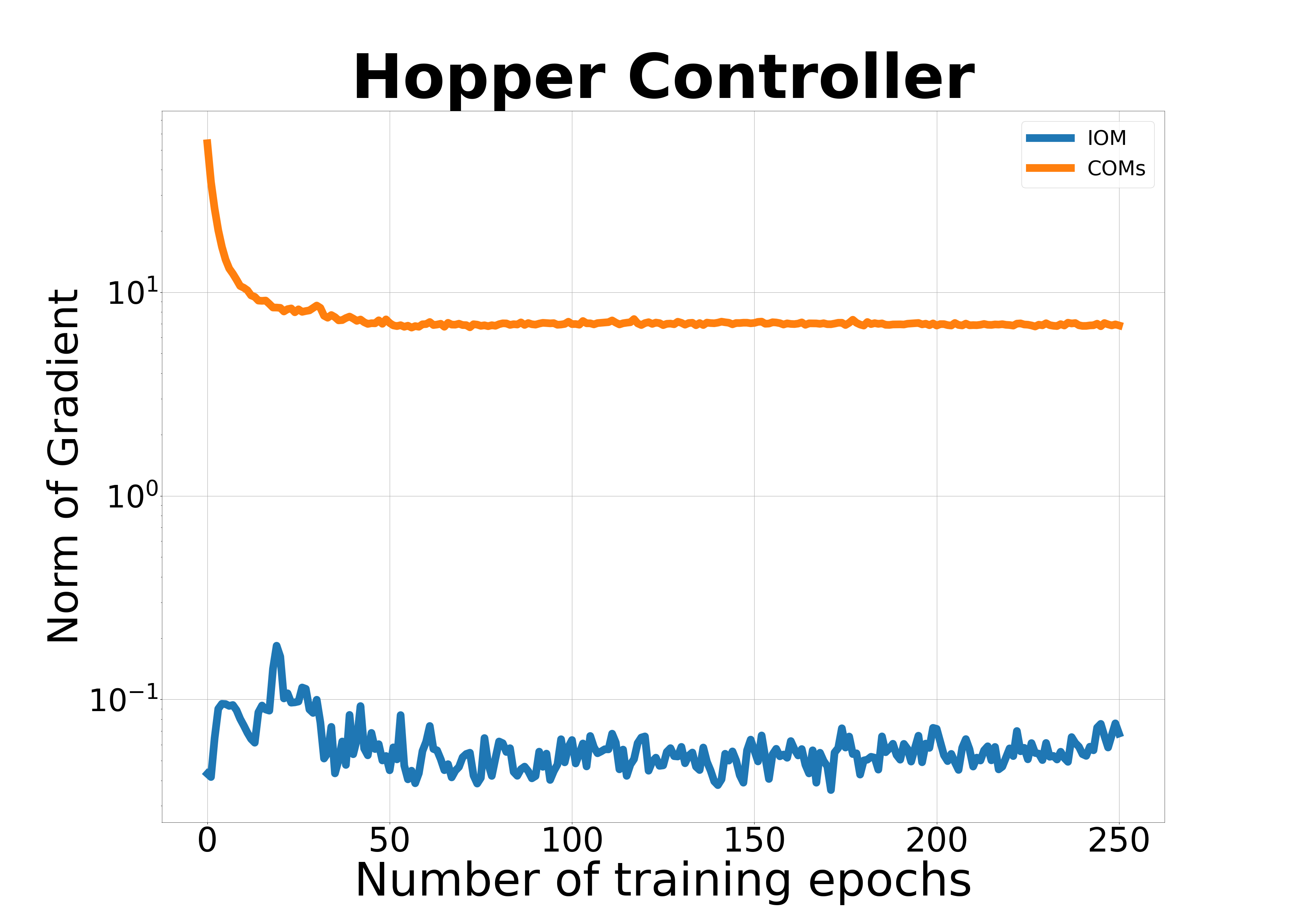}} 
    \caption{\footnotesize {\textbf{Comparison of gradient norm between \methodname\ and COMs during training (log scale).} Note that the gradient norm for COMs is clearly higher, indicating that COMs learn a more non-smooth function than \methodname.}}
    \label{fig:grad_norm}
    \vspace{-0.15in}
\end{figure*}

\pagebreak

\section{{Intuition Behind How \methodname\ Works}}
\label{app:iom_mechanism_intuition}

\begin{figure}[h]
\centering
\includegraphics[width=0.99\linewidth]{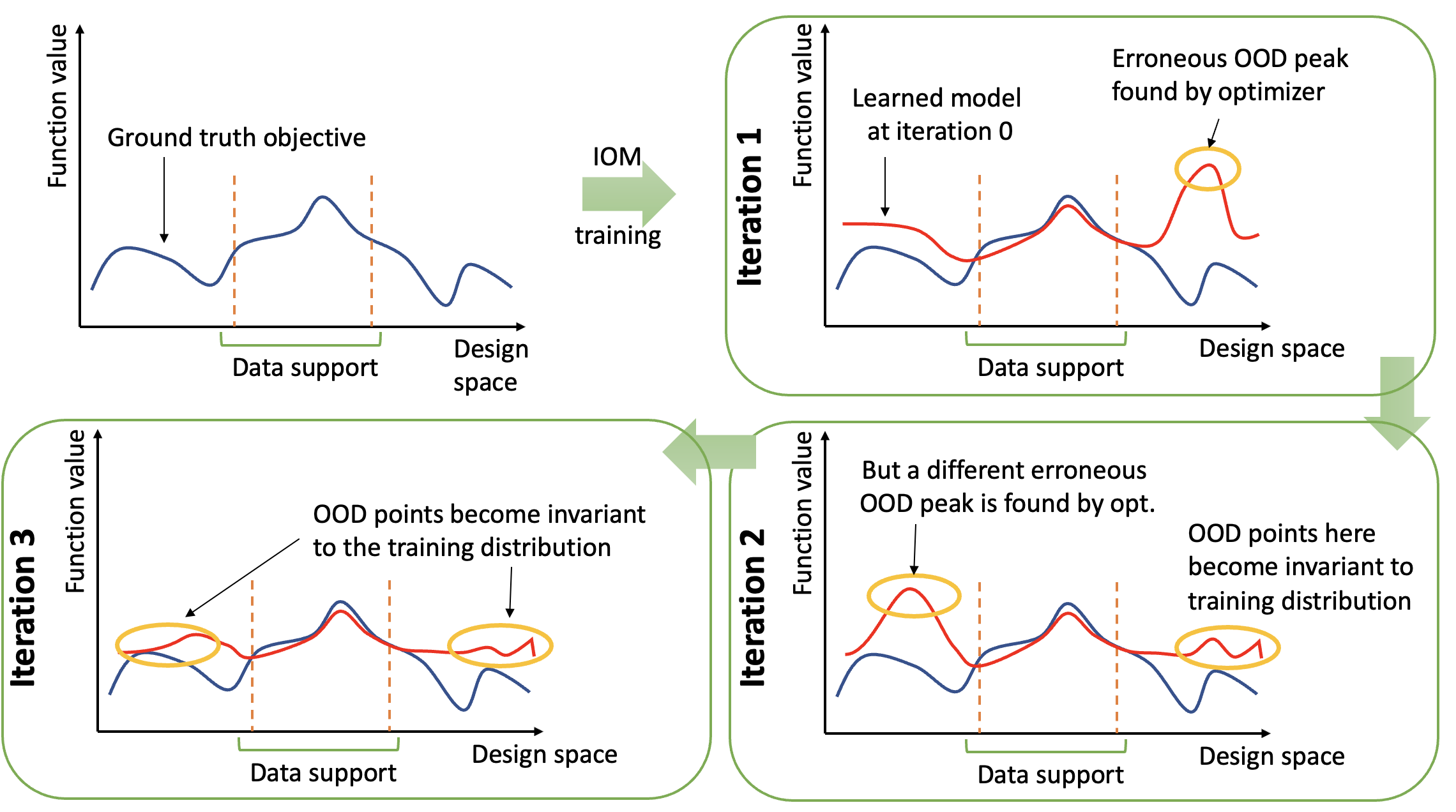}
\caption{\label{fig:intuition_fig} \footnotesize {\textbf{Figure illustrating the intuition behind \methodname.} When we apply the invariance regularizer prescribed by \methodname\ during training, this invariance regularizer would attempt to make the distribution of the representation of the out-of-distribution (OOD) designs found by the optimizer match the distribution of the representations in the training dataset, which would make these points take on a function value that matches the average dataset value. By running this training multiple times, the \methodname\ training procedure would gradually suppress the value of all out-of-distribution erroneous peaks in the learned function, and will enable us to find an optimal design that lies within the training distribution (though not necessarily in the training dataset.)}}
\end{figure}

\end{document}

%% file: defs.tex
\newcommand{\cmt}[1]{{\textcolor{red}{#1}}}
\newcommand{\note}[1]{\cmt{Note: #1}}
\newcommand{\ak}[1]{\textcolor{red}{#1}}
\newcommand{\cmtt}[1]{{\textcolor{blue}{#1}}}

\newcommand{\x}{\mathbf{x}}
\newcommand{\z}{\mathbf{z}}
\newcommand{\w}{\mathbf{w}}
\newcommand{\data}{\mathcal{D}}

\newcommand{\etal}{{et~al.}\ }
\newcommand{\eg}{e.g.\ }
\newcommand{\ie}{i.e.\ }
\newcommand{\nth}{\text{th}}
\newcommand{\pr}{^\prime}
\newcommand{\tr}{^\mathrm{T}}
\newcommand{\inv}{^{-1}}
\newcommand{\pinv}{^{\dagger}}
\newcommand{\real}{\mathbb{R}}
\newcommand{\gauss}{\mathcal{N}}
\newcommand{\norm}[1]{\left|#1\right|}
\newcommand{\trace}{\text{tr}}

\newcommand{\reward}{r}
\newcommand{\policy}{\pi}
\newcommand{\mdp}{\mathcal{M}}
\newcommand{\states}{\mathcal{S}}
\newcommand{\actions}{\mathcal{A}}
\newcommand{\observations}{\mathcal{O}}
\newcommand{\transitions}{\mathcal{T}}
\newcommand{\initial}{\mathcal{I}}
\newcommand{\horizon}{H}
\newcommand{\rewardevent}{R}
\newcommand{\probr}{p_\rewardevent}
\newcommand{\metareward}{\bar{\reward}}

\newcommand{\pihi}{\pi^{\text{hi}}}
\newcommand{\pilo}{\pi^{\text{lo}}}
\newcommand{\ah}{\mathbf{w}}

\newcommand{\loss}{\mathcal{L}}
\newcommand{\eye}{\mathbf{I}}

\newcommand{\model}{\hat{p}}

\newcommand{\pimix}{\pi_{\text{mix}}}

\newcommand{\pib}{\bar{\pi}}
\newcommand{\epspi}{\epsilon_{\pi}}
\newcommand{\epsmodel}{\epsilon_{m}}

\newcommand{\cY}{\mathcal{Y}}
\newcommand{\cX}{\mathcal{X}}
\newcommand{\cZ}{\mathcal{Z}}
\newcommand{\cH}{\mathcal{H}}
\newcommand{\en}{\mathcal{E}}
\newcommand{\cF}{\mathcal{F}}
\newcommand{\be}{\mathbf{e}}
\newcommand{\by}{\mathbf{y}}
\newcommand{\bx}{\mathbf{x}}
\newcommand{\bc}{\mathbf{c}}
\newcommand{\bz}{\mathbf{z}}
\newcommand{\bo}{\mathbf{o}}
\newcommand{\bs}{\mathbf{s}}
\newcommand{\ba}{\mathbf{a}}
\newcommand{\hatf}{\hat{f}}
\newcommand{\tildef}{\tilde{f}}
\newcommand{\bG}{\mathbf{G}}
\newcommand{\bbP}{\mathbb{P}}
\newcommand{\bbE}{\mathbb{E}}
\newcommand{\bbR}{\mathbb{R}}

\newcommand{\ot}{\bo_t}
\newcommand{\st}{\bs_t}
\newcommand{\at}{\ba_t}
\newcommand{\op}{\mathcal{O}}
\newcommand{\opt}{\op_t}
\newcommand{\kl}{D_\text{KL}}
\newcommand{\tv}{D_\text{TV}}
\newcommand{\ent}{\mathcal{H}}
\newcommand{\actionspace}{\mathcal{A}}
\newcommand{\filt}{\mathcal{F}}
\newcommand{\expec}{\mathbb{E}}
\newcommand{\var}{\operatorname{var}}
\newcommand{\piinv}{\pi^{-1}}
\newcommand{\finv}{f^{-1}}
\newcommand{\bigO}{\mathcal{O}}
\newcommand{\atcn}{A_t^{\mathrm{CN}}}
\newcommand{\atts}{A_t^{\mathrm{TS}}}
\newcommand{\bzhi}{\bz^\text{hi}}
\newcommand{\inputspace}{\mathcal{X}}
\newcommand{\yquery}{\tilde{y}^*}
\newcommand{\zquery}{\tilde{\z}^*}
\newcommand{\disc}{\operatorname{Disc}}
\newcommand{\jsd}{D_{\text{JSD}}}
\newcommand{\thetastar}{\theta^\star}
\newcommand{\valdata}{\mathcal{D}_{\text{val}}}
\newcommand{\xopt}{\mathcal{X}_{\text{opt}}}
\newcommand{\lhat}{\widehat{L}}
\newcommand{\tildedata}{\overline{\mathcal{D}}}

\newcommand{\methodname}{\textsc{Iom}}
\newcommand{\conmethodname}{\textsc{Iom}-C}
\newcommand{\Opt}{\mu_\textsc{Opt}}
\newcommand{\Optold}{\textsc{Opt}}
\newcommand{\question}[1]{\cmt{\textbf{Question}: #1}}
\newcommand{\aviral}[1]{\textcolor{red}{#1}}